\documentclass{bmvc2k}

\usepackage{times}
\usepackage{epsfig}
\usepackage{graphicx}
\usepackage{amsmath}
\usepackage{amssymb}
\usepackage{amsthm}
\usepackage{multirow}
\usepackage{url}
\usepackage[utf8]{inputenc}
\usepackage[ruled,vlined]{algorithm2e}
\usepackage{wrapfig}

\usepackage[capitalize,nameinlink]{cleveref}
\newtheorem{theorem}{Theorem}[section]
\newtheorem{proposition}{Proposition}

\newtheorem{lemma}[theorem]{Lemma}
\theoremstyle{definition}

\theoremstyle{remark}

\newcommand{\mres}{\mathbin{\vrule height 1.6ex depth 0pt width
0.13ex\vrule height 0.13ex depth 0pt width 1.3ex}}
\newcommand{\R}{\mathbb{R}}

\DeclareMathOperator*{\argmax}{\arg \max}%

\DeclareUnicodeCharacter{221}{\rightangle}

\title{Fast Convex Relaxations using Graph Discretizations}

\addauthor{Jonas Geiping}{jonas.geiping@uni-siegen.de}{1}
\addauthor{Fjedor Gaede}{fjedor.gaede@wwu.de}{2}
\addauthor{Hartmut Bauermeister}{hartmut.bauermeister@uni-siegen.de}{1}
\addauthor{Michael Moeller}{michael.moellerg@uni-siegen.de}{1}

\addinstitution{
 Department of Electrical Engineering and Computer Science\\
 University of Siegen\\
 Siegen, Germany
}
\addinstitution{
Institut fur Analysis und Numerik,\\
 Westfälische Wilhelms-Universität Münster,\\
Münster, Germany
}

\runninghead{Geiping, Gaede, Bauermeister, Moeller}{Fast Relaxations using Graphs}

\begin{document}

\maketitle

\begin{abstract}
   Matching and partitioning problems are fundamentals of computer vision applications with examples in multilabel segmentation, stereo estimation and optical-flow computation. These tasks can be posed as non-convex energy minimization problems and solved near-globally optimal by recent convex lifting approaches. Yet, applying these techniques comes with a significant computational effort, reducing their feasibility in practical applications. We discuss spatial discretization of continuous partitioning problems into a graph structure, generalizing discretization onto a Cartesian grid. This setup allows us to faithfully work on super-pixel graphs constructed by SLIC or Cut-Pursuit,
    massively decreasing the computational effort for lifted partitioning problems compared to a Cartesian grid, while optimal energy values remain similar: The global matching is still solved near-globally optimal.
   We discuss this methodology in detail and show examples in multi-label segmentation by minimal partitions and stereo estimation, where we demonstrate that the proposed graph discretization can reduce runtime as well as memory consumption of convex relaxations of matching problems by up to a factor of 10. 
\end{abstract}

\section{Introduction}
Matching problems and the closely inter-related minimal partitioning problems are low-level computer vision tasks that build a backbone for a variety of applications such as multi-label segmentation \cite{boykov_graph_2006-1}, stereo estimation \cite{ishikawa_occlusions_1998,ranftl_pushing_2012} and optical flow estimation \cite{horn_determining_1981,brox_high_2004}. However, posing these problems as energy minimization problems leads to non-convex objectives that are difficult to solve. In the last years, it has been demonstrated that functional lifting techniques are very well suited for solving these non-convex minimization problems via convex relaxations in a higher-dimensional space. Unfortunately, despite the precise solutions these methods provide, they incur significant costs in memory and computational effort, due to the high dimensionality of the lifted problem.


\begin{figure}
    \centering
    \includegraphics[width=0.35\textwidth]{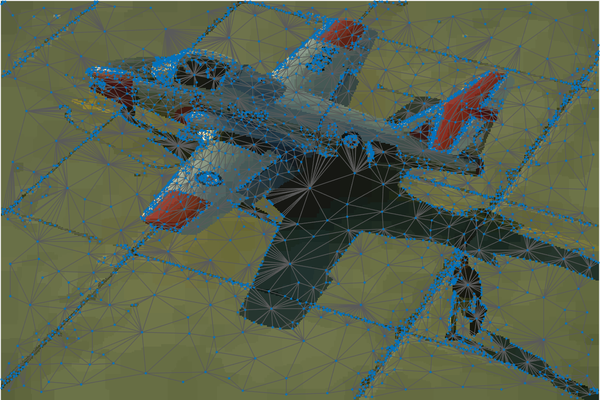}
    \includegraphics[width=0.35\textwidth]{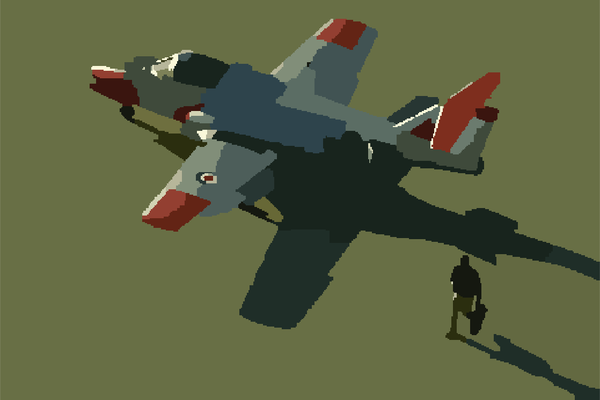}
    \caption{Graph discretization on the left and segmentation by convex relaxation with 32 labels, computed on the graph structure on the right. The segmentation is $3.03\%$ less optimal compared to  relaxation on a full Cartesian grid, but requires only around $4.8\%$ of computation time, the original problem has $4.9$ mio. variables, and the reduced problem only $120$k.}
    \label{fig:teaser}
\end{figure}

Consider the continuous minimal partitions problem, also referred to as piecewise-constant Mumford-Shah problem \cite{mumford_optimal_1989}, which builds the basis of the aforementioned computer vision applications,
\begin{equation}\label{eq:minimal_partitions}
    \min_{\lbrace P_k\rbrace_{k=1}^L} \sum_{k=1}^L \int_{P_k} -f_k(x) + \frac{1}{2}\operatorname{Per}(\Omega, P_k).
\end{equation}
Given a set of $L$ potential functions $f_k$ we are looking for a partitioning of the set $\Omega$ into $L$ non-overlapping partitions $\lbrace P_k \rbrace_{k=1}^L$, i.e. $P_k \cap P_l = \emptyset$ if $k\neq l$ and $\bigcup_{k=1}^L P_k = \Omega$ \cite{chambolle_convex_2012}. Without the regularizing perimeter term, the solution is given by $\argmax_k f_k(x)$ for every $x \in \Omega$, but under inclusion of the second term in \eqref{eq:minimal_partitions}, the perimeter of each partition $\operatorname{Per}(\Omega, P_k)$ is penalized, leading to spatially coherent solutions, where every point is "matched" to a partition, but the global surface energy stays minimal. 

Minimal partitions problems are abundant in imaging. In the discrete setting, they directly relate to the Potts model \cite{potts_generalized_1952} and MRFs \cite{boykov_fast_2001, boykov_graph_2006-1}. This variability hinges on the choice of potential functions $f_k$. For multi-label segmentation, for example, we can consider each $f_k$ to give the prior likelihood that the point $x$ should be assigned label $k$, with the likelihood being computed by model-based approaches \cite{cremers_convex_2011} or returned as the output of a neural network \cite{chen_deeplab:_2016}. On the other hand, considering every partition with label $k$ to encode a displacement of $k$ pixels between two stereo images, the same framework can be used to solve stereo matching problems \cite{zach_fast_2008}. Further immediate examples include optical flow\cite{horn_determining_1981},  scene flow and multiview reconstruction \cite{kolev_integration_2008}.

To be able to solve \cref{eq:minimal_partitions}, we minimize over a set of partitions. This discrete matching problem at the heart of the minimal partitions problem is in general NP-hard. 
In practice, one of the most powerful approaches is the solution of a suitable convex relaxation of the original minimal partitions problem. To do so, the minimization over partitions is first replaced by minimization over their characteristic functions $u_k : \Omega \to \lbrace 0, 1 \rbrace$, satisfying $\sum_{k=1}^L u_k(x) = 1 ~ ~ \forall x \in \Omega$,
\begin{equation}\label{eq:minimal_partitions_binary}
    \min_{\lbrace u_k\rbrace_{k=1}^L} \ \sum_{k=1}^L \ \int_\Omega -f_k(x)u_k(x) + \int_\Omega |Du_k|,
\end{equation}
where we have equivalently replaced the perimeter of a set by the total variation of its characteristic function, see \cite{chambolle_convex_2012}. In a next step, the functions $u_k$ are relaxed to take values in the full interval $[0,1]$, either directly \cite{zach_duality_2007, lellmann_convex_2009}, or by jointly deriving tighter convex reformulations of the regularizer \cite{chambolle_convex_2012}.
Relaxation approaches often lead to near-optimal solutions with high fidelity \cite{pock_convex_2008, strekalovskiy_tight_2011}, yet the computational effort amounts to solving a non-smooth, non-strongly convex optimization problem over all functions $u_k$, each of which is usually discretized to be as large as the given potential functions. Especially when $k$ is large, memory costs quickly become impractical for computer vision applications.

In this work we hence consider strategies to remediate the computational costs of functional lifting techniques, without majorly impeding their global matching capabilities. As illustrated in Figure \ref{fig:teaser}, we propose to first discretize the problem on a precomputed graph structure instead of a Cartesian grid, and then solve the convex relaxation on the graph. This strategy leads to significant computational advantages while sacrificing almost no accuracy in terms of the energy of the final solution.

\section{Related Work}
The minimal partitions problem is a prime example of energy minimization methods in computer vision \cite{mumford_optimal_1989,rudin_nonlinear_1992,burger_guide_2013}, which have found widespread use. In the context of convex relaxations of these partition problems, there have been works in as diverse applications such as stereo estimation \cite{pock_convex_2008,pock_global_2010,werlberger_efficient_2011,ranftl_pushing_2012,ranftl_minimizing_2013}, optical flow \cite{strekalovskiy_tight_2011,strekalovskiy_convex_2014}, segmentation \cite{zach_what_2012,zach_duality_2007,lellmann_convex_2009,chambolle_convex_2012},and optimization on manifolds \cite{lellmann_total_2013}, with algorithmic improvements such as \cite{souiai_entropy_2015}. 

Previous work discusses the optimal discretization of the continuous label dimension \cite{mollenhoff_sublabel-accurate_2016,laude_sublabel-accurate_2016-1,mollenhoff_sublabel-accurate_2017}, reducing the computational effort of functional lifting in a variety of applications. In our work, we discuss an orthogonal direction of research, as we are discussing compact discretizations of the image space. 

The choice of efficient discretization of the input image data is directly related to superpixel approaches, e.g. \cite{achanta_slic_2012,uziel_bayesian_2019}. Their general idea is to generically reduce the computational complexity of any (pixel-based) numerical algorithm, by locally grouping pixels of similar color to larger superpixels. The most prominent algorithm in current practice is SLIC (Simple Iterative Linear Clustering) \cite{achanta_slic_2012}.  
Ideally, the superpixel setup should also be chosen by an appropriate minimization procedure that adheres object edges. However, edge adherence is often costly. An interesting exception is the Cut-Pursuit algorithm \cite{landrieu_cut_2016,landrieu_cut_2017}, which solves total variation minimization and related problems in a fast sequence of binary graph cuts, making it competitive as a discretization step and leading to boundaries that better adhere with minimal partitions. This approach has been successfully applied in practice in such works as \cite{landrieu_large-scale_2017,guinard_piecewise-planar_2019} and we will contrapose a superpixel structure generated by Cut Pursuit with one generated by SLIC in our main comparison to a Cartesian grid.

\section{Graph Discretizations for Convex Relaxations}
\subsection{Preliminaries}
Let us first introduce our general notation. For the discrete setup we consider an undirected graph structure is defined by its vertices $V$, edges connecting vertices $E \subset V \times V $ and weights of these edges $w \in \R^{|E|}$. We refer to \cite{elmoataz_nonlocal_2008} for details.

The continuous minimal partition problem requires the definition of the total variation of a function $u = (u_1, \dots, u_L) \in L^1(\Omega, \R^L)$ as
\begin{align*}
    TV(u) = \sup  \bigg\lbrace &- \int_\Omega \sum_{k=1}^L u_k(x) \operatorname{div}\mathbf{p}_k(x) \ dx, ~\ 
    \mathbf{p} \in C^1_c(\Omega, \R^{d \times L}), ~\  \sum_{k=1}^L |\mathbf{p}_k(x)|^2\leq 1  \bigg\rbrace,
\end{align*}
Note that the above definition reduces to $TV(u) = \sum_{k=1}^L \int_\Omega |\nabla u_k(x)|~dx$ for smooth $u$. We define $u$ to be an element of the space of bounded variation $SBV(\Omega,\R^L)$ if $TV(u)$ is finite. We can then identify this value with the mass of the distributional derivative $Du$, i.e. $\int_\Omega |Du| = TV(u)$. 
The bounded Radon measure $Du$ can be decomposed \cite[Thm. 10.4.1]{attouch_variational_2006} into
\begin{equation}\label{eq:decomposition}
    Du = \nabla u \ \mathcal{L}^d + Cu + (u^+-u^-) \otimes \nu_u \mathcal{H}^{d-1} \mres J_u,
\end{equation}
where $\mathcal{L}^d$ is the Lebesgue measure, $J_u$ is the jump set of $u$, where $u^+\neq u^-$, i.e. the values at the boundary differ, $\nu_u$ the normal of the boundary and $Cu$ a remainder Cantor part. 
In the following we will consider functions $u \in SBV(\Omega, \R^L)$, which is the space of functions for which $Cu=0$ \cite{ambrosio_functions_2000,mollenhoff_sublabel-accurate_2017}. 
We further define the perimeter of a measurable set $P \subset \Omega$, $\operatorname{Per}(\Omega, P)$, in turn by the total variation of its characteristic function $\chi_P:\Omega \to \R$  
\cite{chambolle_convex_2012}, the boundary of a set as $\partial S = \bar{S}  \setminus \operatorname{int}(S)$, and the length of the boundary $\Gamma_{k,l} = \partial P_k \cap \partial P_l$ between two sets $P_k, P_l$ via
\begin{equation}
    \lvert \Gamma_{k,l} \rvert = \mathcal{H}^{d-1}(\Gamma_{k,l}), 
\end{equation}
where, again, $\mathcal{H}^{d-1}$ denotes the $d-1$-dimensional Hausdorff measure. These definitions allow us to examine the continuous boundary of shapes. Refer to \cref{fig:visualization}, where these continuous objects are marked in red.
\subsection{Graph Discretization}\label{sec:theory}
We are interested in solving the continuous minimal partition problems \cref{eq:minimal_partitions_binary} numerically. To do so we need to translate the problem into the discrete setting. To take a step from the continuous definitions to a discrete problem, we make use of the fact that we expect solutions $u^*$ to the minimal partitioning problem to be piecewise constant with a finite number of pieces. A good discretization to a finite setting that mimics this piecewise constant structure exactly. 
\emph{We hence represent the discretization by a graph of candidate constant sets, the nodes of which represent each separate constant piece and where neighboring pieces are connected by edges in the graph}. After solving the matching problem on this discrete graph, the final solution $u^*$ can be reassembled by assigning to each constant piece its matched value according to the respective value of the node that represents it. This setup is sketched in \cref{fig:visualization}.
\begin{figure}
    \begin{subfigure}
         \centering
         \includegraphics[width=0.35\textwidth]{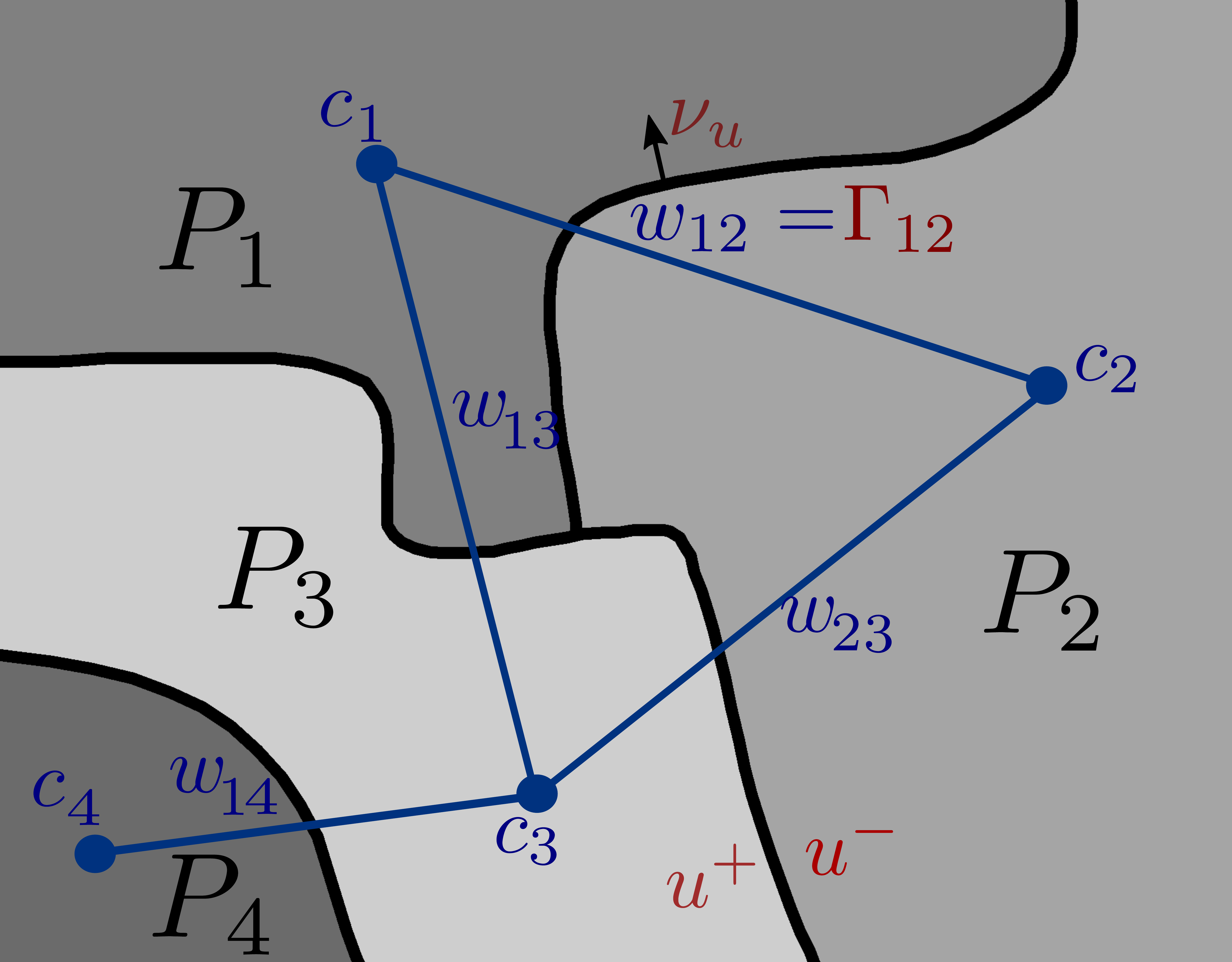}
     \end{subfigure}
     \begin{subfigure}
         \centering
         \includegraphics[width=0.35\textwidth]{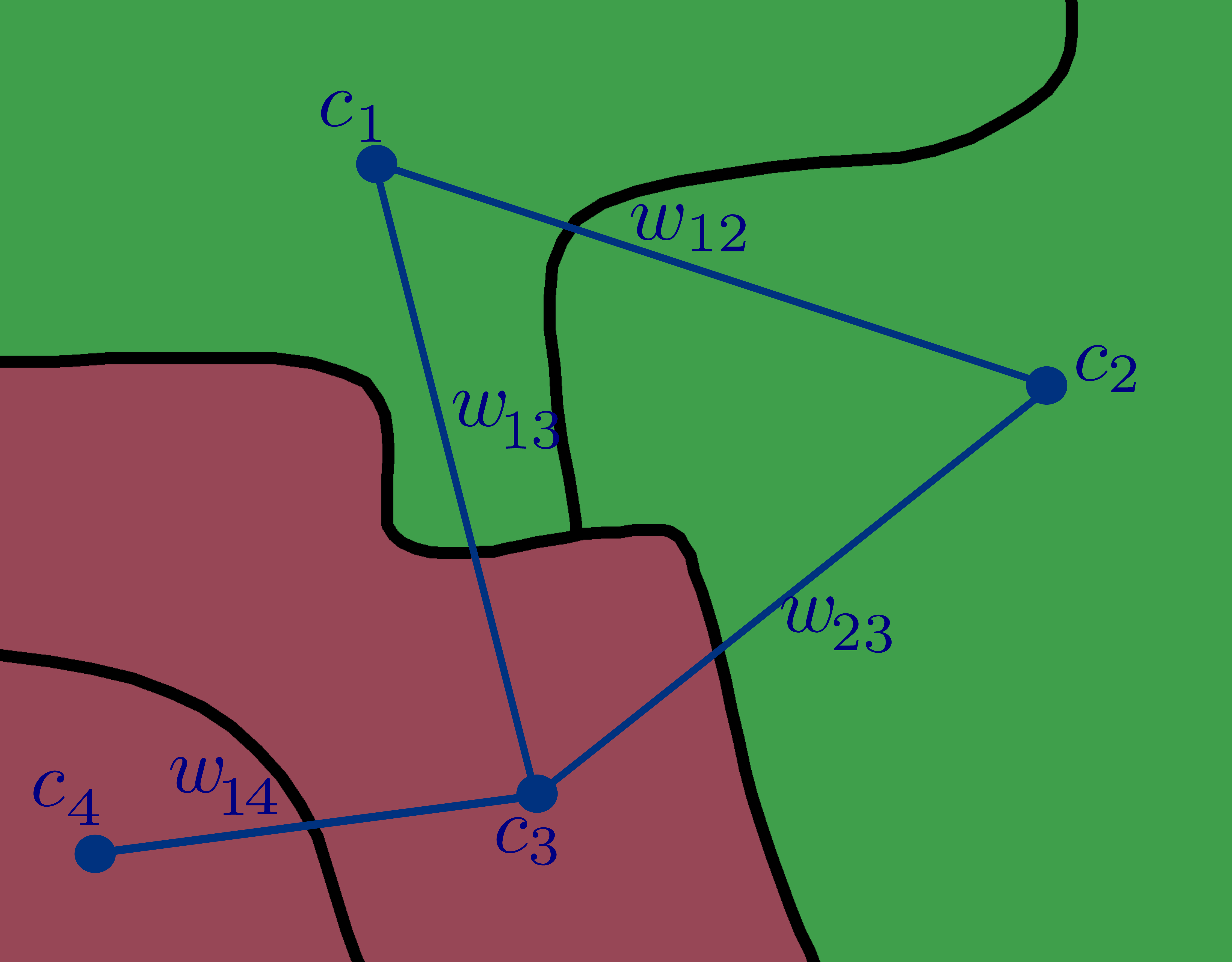}
     \end{subfigure}
    \centering
    \caption{Sketch of the discretization process for a graph-based discretization. \textit{Left:} The underlying continuous function $u \in SBV_\Pi(\Omega,\R^L)$ is pictured in black, with piecewise-constant partitions $\Pi=\{P_1,P_2,P_3,P_4\}$ shown in gray as well as the discrete graph structure in blue. \textit{Right:} A minimal partitions problem with two potentials is solved on this graph structure. Pictured is the solution $u^*$ which now corresponds to a piecewise constant solution (in red and green).}
    \label{fig:visualization}
    \vspace{-0.5cm}
\end{figure}

Note that this is a generalization of a classical discretization to a Cartesian grid. Placing a continuous function on a pixel grid corresponds to claiming that the function is piecewise-constant on every image pixel - hence the boundaries of the solution $u^*$ to the minimal partitions problem will be a subset of the boundaries imposed by the image pixels.

In slight generalization of the minimal partitions problem in \cref{eq:minimal_partitions_binary} we now discuss the type of continuous functionals $F:SBV(\Omega, \R^L)\to \R$, that we want to represent discretely. Due to the discontinuities of $SBV$, we define different components of $F$ on the continuous parts and the jump parts $J_u$ of \cref{eq:decomposition}:
\begin{align}\label{eq:functional}
\begin{split}
      F(u) = \int_{\Omega \setminus J_u} \Phi(x, u(x), \nabla u(x)) \ dx
    + \int_{J_u} \kappa\left(|u^+-u^-|\right)|\nu_u| d\mathcal{H}^{d-1},  
\end{split}
\end{align}
where $\nabla, u^+, u^-$ refer to the decomposition detailed in \cref{eq:decomposition}. $\Phi:\Omega \times \R^L \times \R^{L \times d}\to \R$ is a function defined away from the jump set of $u$, while $\kappa:\R \to \R$ is a concave function measuring the jump penalty with $\kappa(0)=0$ \cite{chambolle_convex_2012,mollenhoff_sublabel-accurate_2017}. We can consider the first term to be a generalized data term, and the second as a (jump)-regularizer.

To connect the continuous formulation of \cref{eq:functional} to a discrete setting we define the discretization as a finite set of candidate sets $\Pi = \{P_i \subset \Omega \mid P_i \cap P_j = \emptyset, \ \forall j \neq i \}$ with $M = \vert \Pi\vert $ partitions. The continuous function $u \in SBV(\Omega,\R^L)$ is assumed to be constant on every partition, so that we can denote its value on partition $P_i\in \Pi$ by a vector $c_i \in \mathbb{R}^L$. Thus, $u(x) = c_i$ for every $x\in P_i \subset \Omega$. 
The partition $\Pi$ can be represented by a set of nodes $V = \{1,\ldots, M\}$ where each node corresponds to a segment $P_i\in \Pi$. Furthermore we can describe every boundary between sets $P_i$ and $P_j$ as $\Gamma_{ij}$ and by that define an edge set $E\subset V\times V$ as $E = \{(i,j) \in V\times V \mid \lvert \Gamma_{ij} \rvert > 0, i \neq j \}$. Note, that the perimeter of some partition $P_i \in \Pi$ is given by $\textnormal{Per}_{P_i} = \sum_{(i,j)\in E}\vert \Gamma_{ij}\vert$.

Let us assume that our desired solution $u^*$, which minimizes \cref{eq:functional}, is piecewise constant. More formally, given some partition $\Pi$ let us write $u \in SBV_\Pi(\Omega, \R^L)$ to denote continuous functions in $SBV$ which are piecewise constant on the regions in $\Pi$, and assume $u^*\in SBV_\Pi(\Omega, \R^L)$. This implies that the jump set $J_u$ is a subset of $\cup_{(i,j)\in E} \Gamma_{ij}$ and that $\Omega\setminus J_u$ is a subset $\cup_{i\in V} P_i$, or, in other words, the discrete partitioning by $\Pi$ is able to represent the continuous structure of $u^*$.



Under the above assumption we can restrict the minimization of $F$ over all functions $u \in SBV(\Omega, \mathbb{R}^L)$ to those in $SBV_\Pi(\Omega, \R^L)$ which allows to simplify \cref{eq:functional} to a problem in which merely the values $c_i$ inside the piecewise constant regions are the unknowns. Let us discuss the three main components of \cref{eq:functional} separately.

\noindent \textbf{Data Term:} \\ Considering $F$ for any $u \in SBV_\Pi(\Omega, \R^L)$ allows us to rewrite the first term of \cref{eq:functional} as
\begin{align}\label{eq:dataterm}
    \begin{split}
        K(u) = \int_{\Omega \setminus J_u} \Phi(x, u(x), \nabla u(x)) \ dx 
             = \sum_{i=1}^M ~ \int_{P_i} \Phi(x, c_i, 0) \ dx 
             =: K_\Pi(c)
    \end{split}
\end{align}
which is the discrete representation $K_\Pi(c):\R^{M \times L} \to \R$ of this term that mere depends on the values $c_i$. For linear data terms such as in \cref{eq:minimal_partitions_binary}, i.e. $\Phi(x, u(x), \nabla u(x)) = \sum_{k=1}^L f_k(x) u_k(x) = f_k(x) (c_k)_i$ for $x\in P_i$, this is further simplified to 
\begin{equation}\label{eq:scalar_product}
    K_\Pi(c)=\sum_{k=1}^L \sum_{i=1}^M ~ (c_i)_k \int_{P_i} f_k(x) \ dx  = \sum_{i=1}^M ~ \langle c_i, \tilde{f}_k \rangle 
\end{equation}
with $\tilde{f}_k = \left(\int_{P_i} f_k(x) dx\right)_{k=1}^L \in \R^{L}$.

\newpage
\noindent \textbf{Regularization Term:}\\ For the jump regularization, we can write $R(u)$ for any $u \in SBV_\Pi(\Omega, \R^L)$ as
\begin{align}\label{eq:regularizer}
    \begin{split}
        R(u) &=  \int_{J_u} \kappa\left(|u^+-u^-|\right) \  d\mathcal{H}^{d-1}
            =\sum_{(i,j) \in E} ~ \int_{\Gamma_{ij}} \kappa\left(|c_i-c_j|\right)        \ d\mathcal{H}^{d-1} \\
            &=\sum_{(i,j) \in E} ~ \kappa\left(|c_i-c_j|\right)\int_{\Gamma_{ij}}  d\mathcal{H}^{d-1} 
            =\sum_{(i,j) \in E} ~ w_{ij}~ \kappa\left(|c_i-c_j|\right)
             =: R_\Pi(c)
    \end{split}
\end{align}
identifying the weights $w_{ij} = \int_{\Gamma_{ij}} d\mathcal{H}^{d-1} =  \lvert \Gamma_{ij} \rvert$. 
With this weighting we can define the weighted finite graph $G = (V, E, w)$ as the discrete graph structure with which any continuous function $u \in SBV_\Pi(\Omega, \R^L)$ can be represented. Note that if $\kappa$ is the identity, then $R_\Pi(c)$ is equivalent to graph total variation of $c$ (cf. \cite{gilboa_nonlocal_2008,bresson_non-local_2008-1}).

\noindent  \textbf{Constraint Set:}\\
We are further carrying a constraint set when minimizing the minimal partitions problem. However, both constraints are pointwise and therefore straight forward to relate to constraints on $c_i$, i.e., the constraint set
directly translates to 
\begin{align*}
C_\Pi = \big\{ c ~ \big| ~ (c_i)_k \in [0,1],~\sum_{k=1}^L (c_i)_k = 1,  \forall i\big\}.
\end{align*}
Interestingly, the above restriction from the minimization of $F$ over $SBV(\Omega, \R^L)$ to its minimization over $SBV_\Pi(\Omega, \R^L)$ (which translates into the minimization of $F_\Pi$ over $c\in C_\Pi$) remains valid as long as the jump set of the true solution is a subset of the jumps in the partition $\Pi$, independent of what exactly the "super" jump-set of $\Pi$ is. Let us formalize this result:
\begin{figure*}
    \centering
    \includegraphics[width=0.26\textwidth]{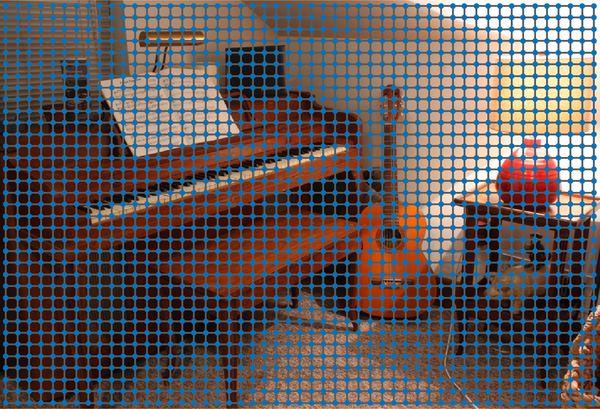}
    \includegraphics[width=0.26\textwidth]{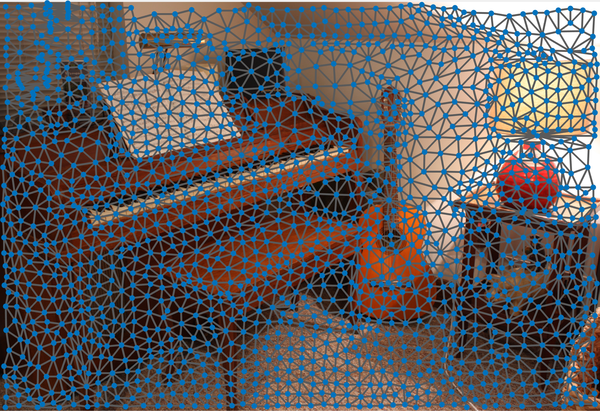}
    \includegraphics[width=0.26\textwidth]{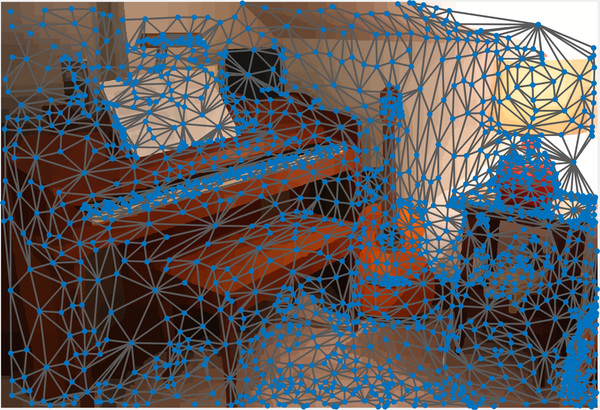}\\
    \includegraphics[width=0.26\textwidth]{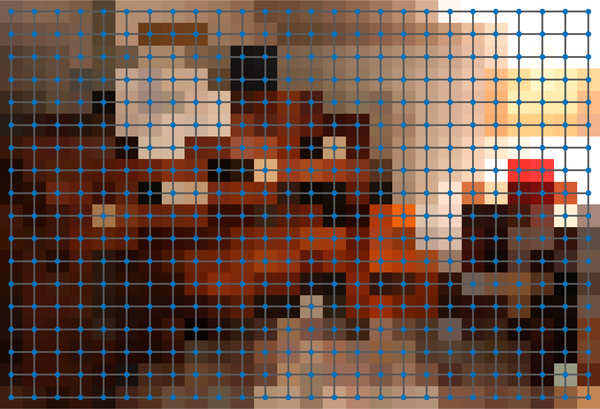}
    \includegraphics[width=0.26\textwidth]{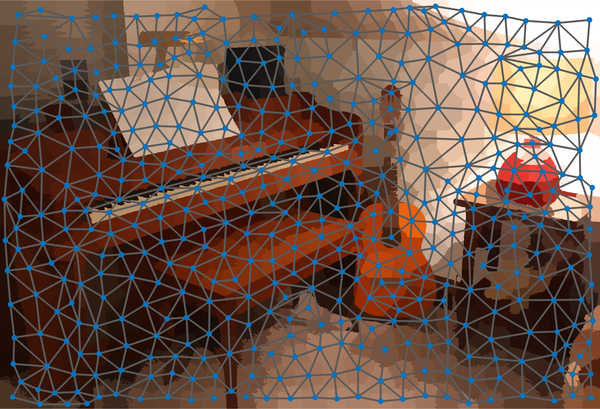}
    \includegraphics[width=0.26\textwidth]{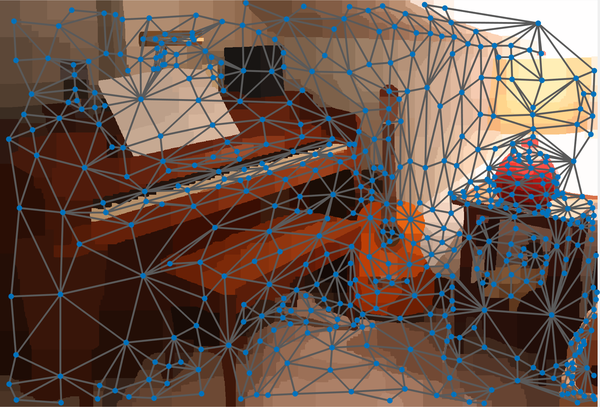}
    \caption{From left to right: Grid Sampling, SLIC Superpixels and $L^0$ Cut-Pursuit. Images from the Middlebury dataset \cite{scharstein_high-resolution_2014}. The top row shows a fine discretization into the same number of nodes for every method, whereas the lower row shows a coarse discretization with the same number of nodes for every method.}
    \label{fig:superpixel_comparison}
\end{figure*}
\begin{proposition}\label{prop1}
Assume a discretization $\Pi$ and its assorted partitions $P_i$ to be given.
Let $u^*$ be a minimizer to the continuous problem \cref{eq:minimal_partitions_binary} for given potentials $f_k$. If the jump-set $J_{u^*}$ of $u^*$ is a subset of the jump set of $\Pi$ given as the boundaries $\cup_{(i,j)\in E} \Gamma_{ij}$, then
\begin{equation*}
    \min_{u\in C} F(u) = \min_{c\in C_\Pi} F_\Pi(c),
\end{equation*}
for the discrete energy $F_\Pi = K_\Pi + R_\Pi$,
i.e. the continuous minimum $F(u^*)$ is equal to the minimum $F_\Pi(c^*)$ of the discrete energy of $F_\Pi$ under the constraints $C_\Pi$.
\end{proposition}
\begin{proof}
See appendix.
\end{proof}

\cref{prop1} shows that if the jump set of $u^*$ is contained in $Pi$, then the exact optimum $u^*$ of the continuous problem can actually by found by computing a discrete solution $c^*$ of the function $F_\Pi$ numerically on a finite graph
Practically however, we now need to find some partition $\Pi$ that approximates (or ideally overestimates) the true jump set $J_{u^*}$, but consists of a limited number of segments. 
%
%
%
On a Cartesian grid, the equivalent operation is to subsample the image, result in the "superpixels" seen in \cref{fig:superpixel_comparison} on the left, which are not well aligned with edges in the images. However approaches such as SLIC (middle) or Cut-Pursuit (right, in the variant of \cite{tenbrinck_variational_2019}) are more adept at finding a superset of candidate partitions.
\begin{figure*}
    \centering
   \includegraphics[width=0.32\textwidth]{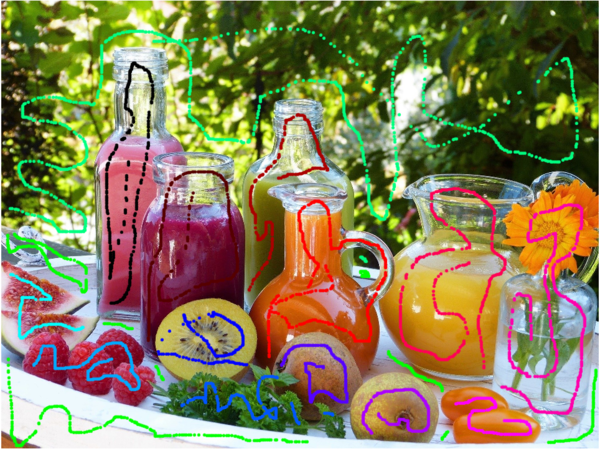}
   \includegraphics[width=0.32\textwidth]{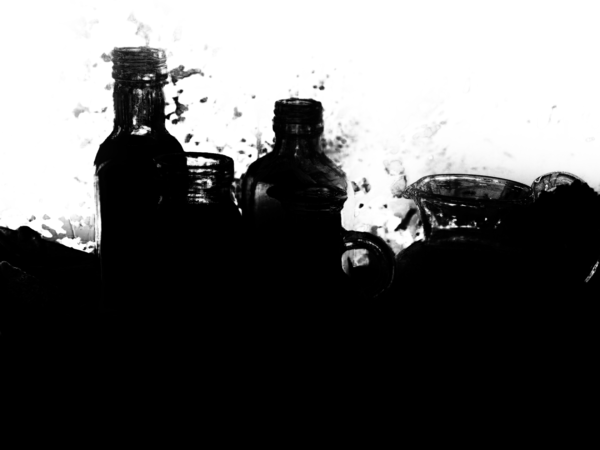}
   \includegraphics[width=0.32\textwidth]{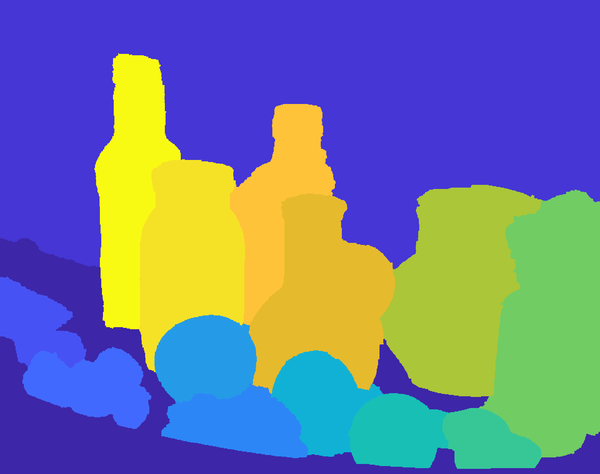}
    \caption{Illustrating the use of model-based segmentation methods: The user scribbles different objects to be segmented (left) from which a unary data term $f$ is generated, e.g. by approaches like \cite{nieuwenhuis_survey_2013} or a pointwise neural network. As the unaries are insufficient for a good segmentation (illustrated by the fact that a thresholding of the unaries shown in the middle does not segment the background well), the proposed framework offers an efficient approach to obtain accurate segmentations as shown on the right.}
    \label{fig:segmentation}
\end{figure*}
\section{Numerical Evaluation} 
This section focuses on evaluating the proposed approach. We discuss examples in segmentation and stereo estimation. For segmentation we show a practical example, where the approach is used to align the output of a pixelwise neural network. We then follow up with a detailed comparison of graphs generated by SLIC, Cut Pursuit and subsampling. 
\subsection{Segmentation}

Multi-label segmentation is a central application of minimal partition problems, having been discussed in the continuous setting in works such as  \cite{chan_active_2001,pock_convex_2009} and widely studied in discrete methods such as \cite{boykov_fast_2001,kolmogorov_what_2004}. 
To apply multilabel segmentation we use the model described in \cref{eq:minimal_partitions}, which can be recovered from \cref{eq:functional} by setting $\kappa=\operatorname{Id}$ and choosing the $L^1$ norm for $|\cdot|$, leading to an anisotropic penalty of the jumps. We solve the discrete matching on the graph by a preconditioned primal-dual algorithm as discussed in \cite{tenbrinck_variational_2019}. Relaxed solutions are matched to corresponding partitions with maximal argument.

\noindent \textbf{Use Case:}\\
As one application scenario, imagine a user wants to segment an image by marking the objects to be segmented with scribbles, see \cref{fig:segmentation} on the left. Once the scribbling is complete we train a tiny pixelwise fully connected network on classifying the scribbled pixels correctly. 
The output of this network provides us with pixelwise features as shown in \cref{fig:segmentation} in the middle. Globally matching these features  with 15 labels on the full grid requires 6 GB of memory, whereas the proposed approach reduces the memory requirements to 0.2 GB due to the graph construction and needs only 13\% of computation time in total to yield the segmentation shown in \cref{fig:segmentation} on the right, which is precise enough to conduct various image manipulations such as inserting or removing some of the bottles or fruits. 
\begin{figure}
    \centering
    \includegraphics[width=0.26\textwidth]{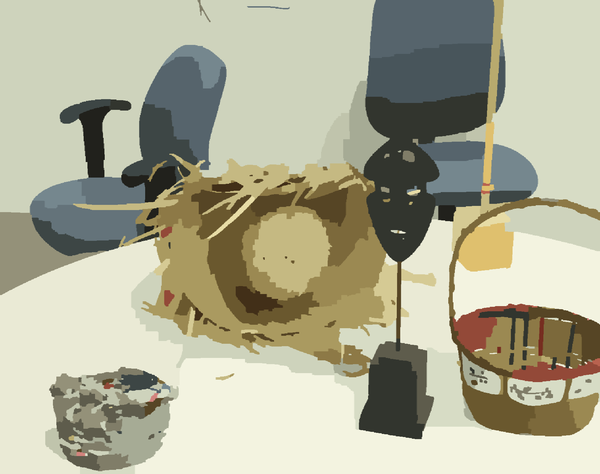}
    \includegraphics[width=0.26\textwidth]{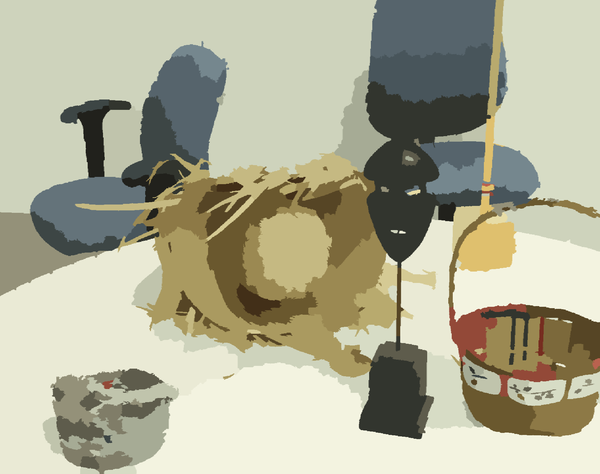}
    \includegraphics[width=0.26\textwidth]{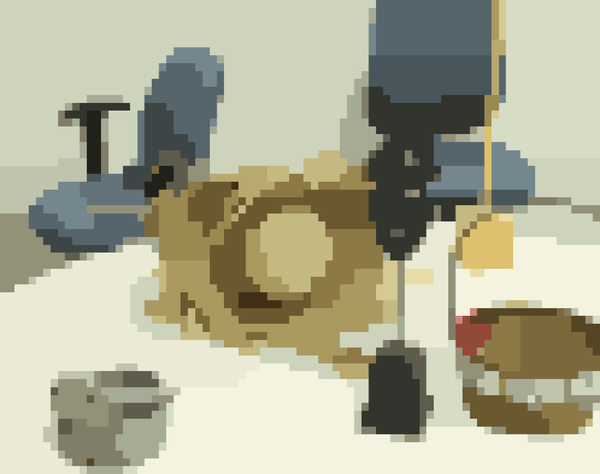}\\
    \includegraphics[width=0.26\textwidth]{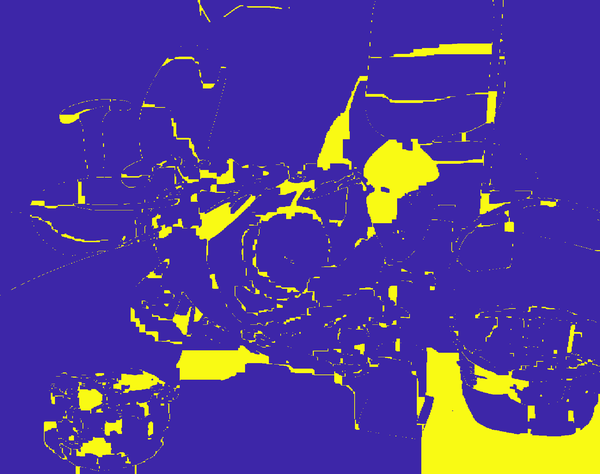}
    \includegraphics[width=0.26\textwidth]{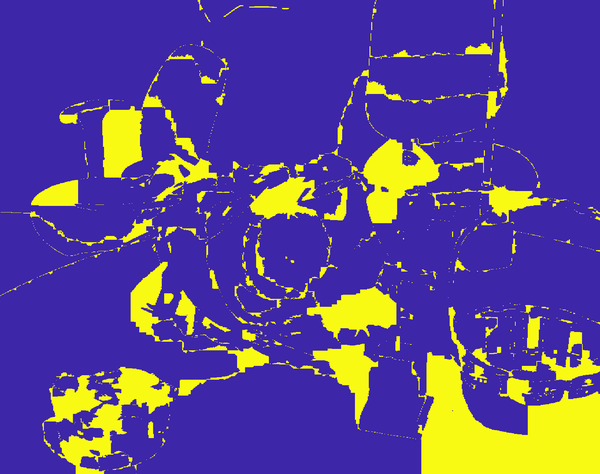}
    \includegraphics[width=0.26\textwidth]{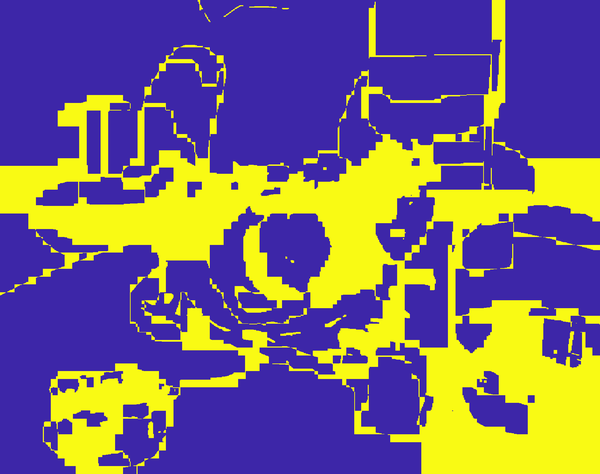}\\
    \caption{Qualitative comparison of matching quality for the example of image segmentation from a given set of pixel features. From left to right:$L^0$-CP graph, the SLIC graph, and a rectangular grid (right), all with the same number of vertices. The graphs are constructed as described in \cref{sec:theory}. The top images show the final minimal partition result. The bottom images show the errors compared to the minimal partition computed on the full pixel grid, where yellow marks regions that are matched differently compared to the ground truth matching of the full image grid. }
    \label{fig:qualitative_comparison_segmentation}
    \vspace{-0.25cm}
\end{figure}

\noindent \textbf{Quantitative Analysis:}\\
To analyze a wide range of images with canonical potentials, we turn to cartooning, i.e. multi-label segmentation with a fixed set of target colors chosen by a k-means selection. 
\Cref{fig:qualitative_comparison_segmentation} visualizes the result of the minimal partition problem for our graph discretization via $L^0$-CP, a graph constructed via SLIC, and a subsampling of the pixel grid, all with the same number of vertices. 
Checking the error maps on the bottom row of  \cref{fig:qualitative_comparison_segmentation} 
we see that both superpixel methods lead to solutions that closely match the solution at the finest level, while the subsampling is comparatively error-prone. 
The $L^0$-CP constructed discretization outperforms the SLIC-based discretization, due to its closer adherence to image edges. 
Evaluating the gained efficiency in terms of time and in times of vertices in \cref{fig:plot_comparison} shows that this behavior leads to stable improvements over a wide range of graph discretization steps, energy values can be matched very closely using the superpixel-based graph discretization.
\begin{figure*}
    \centering
    \includegraphics[width=0.44\textwidth]{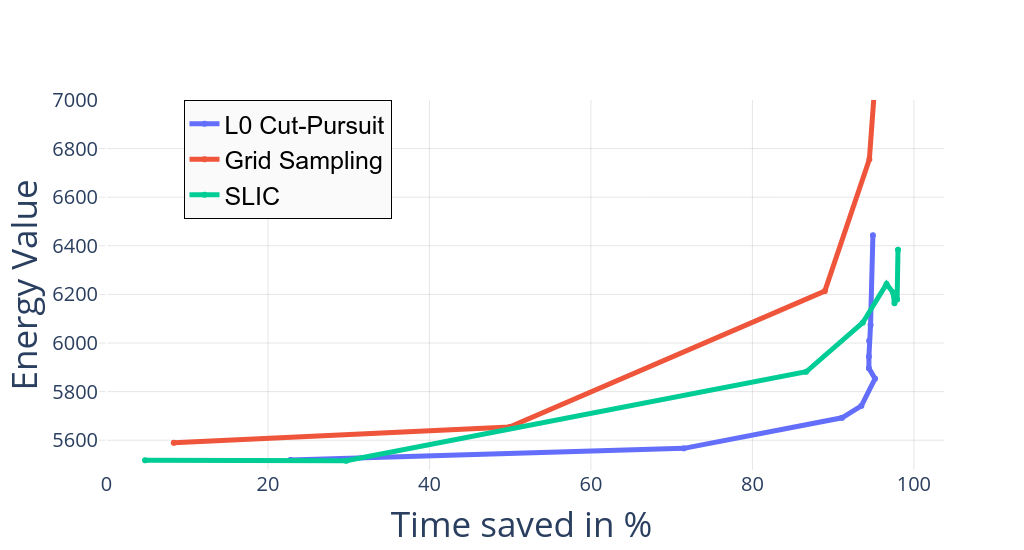}
    \includegraphics[width=0.44\textwidth]{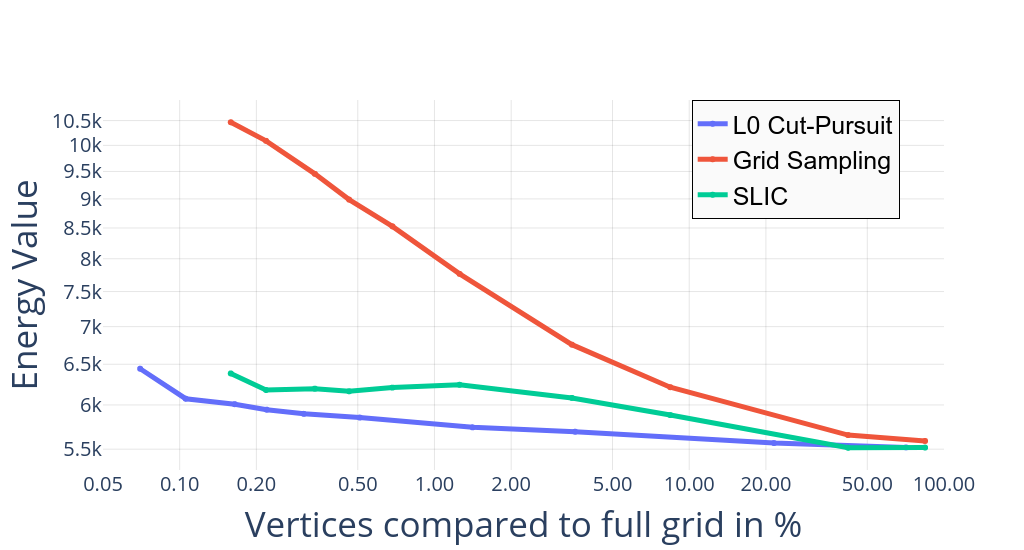}
    \caption{Computing the minimal partition on a well chosen graph discretization is much more efficient than computing it on the full grid. The y-axis in both plots denotes the energy value of the ground truth solution, which is computed on the full grid. Left: Time saved vs ground truth plotted vs the matching energy of the minimizer. Right: The number of nodes compared to the partition energy of the minimizer. }
    \label{fig:plot_comparison}
    \vspace{-0.25cm}
\end{figure*}
In \cref{tab:inset} we compare the three methods for three different examples.
We find significant savings in runtime and memory, while staying close to the original energy, observing that graph-based discretization leads to a significant improvement in accuracy, compared to computing the segmentation on a downsampled grid, and further that using the $L^0$-CP superpixels leads to the most efficient final result, even though the computation of these superpixels itself takes more time than SLIC (the time/memory to compute superpixels and construct the graph is factored into all measurements we consider).
\begin{table}
\footnotesize
\centering
\begin{tabular}{|c|c|c|c|c|c|}
            \hline
            \multirow{2}{*}{Ex.}  & \multirow{2}{*}{Methods}  & Red.  & Time     &  \multirow{2}{*}{Mem.} & Energy\\ 
                                        &                           & Rate                 &     Saved  &               &   Offset \\\hline
            \multirow{3}{*}{1}  & $L^0$-CP               &   $\mathbf{31\%}$     &   $\mathbf{79\%}$   &  $\mathbf{15MB}$            & $\mathbf{0.73 \%}$   \\
                                        & SLIC                  &   $ 41\% $ &      $41\%$              &        16MB               &$ 1.18\% $   \\
                                        & sampling                      &   $74\%$        &  $22\%$            &  32MB   &       $5.77\%$          \\\hline
            \multirow{3}{*}{2}  & $L^0$-CP               &   $\mathbf{3.6\%}$        &  $\mathbf{87\%}$   &   $\mathbf{12MB}$           & $\mathbf{3.5 \%}$   \\
                                        & SLIC                  &   $ 8.42\% $& $\mathbf{87\%}$              &        32MB             &$ 6.29\% $   \\
                                        & sampling                      &   $8.42\% $       & $\mathbf{87\%}$            & 27MB    &             $13.33\%$          \\\hline
            \multirow{3}{*}{3}  & $L^0$-CP               &   $\mathbf{6.2\%}$        &$\mathbf{83\%}$   & $\mathbf{463MB}$         & $\mathbf{2.1\%}$   \\
                                        & SLIC                  &   $ 14\% $  & $79\%$              &  1144MB                   &$ 4.43\% $   \\
                                        & sampling                      &   $14\% $      & $82\%$            &  2976MB   &             $3.79\%$          \\\hline 
            
            \end{tabular}
\caption{Different scores for three examples. Shown are the ratio of time saved and the ratio of energy mismatch. The baseline method (a full image grid) uses 301, 684 and 6113 MB for each experiment. \label{tab:inset}}
\end{table}

\subsection{Stereo Matching}
For the task of estimating disparities in stereo images the problem setting is different from that of segmentation tasks. While we want to reconstruct discrete labels for segmentation the estimated disparities between images live in a continuous range and therefore need dedicated treatment. The binarized vector structure of discrete segmentation labels ideally has to be translated to a continuously metricized label space. 
Assuming piecewise constant disparities in natural images it is still possible to transfer the graph reduction ideas to stereo estimation as proposed in \cite{pock_global_2010} and the sublabel accurate setting of \cite{mollenhoff_sublabel-accurate_2016}. 
As the data term for stereo matching can be expressed as a cost-vector defined for each pixel, we need to find a sensitive scalar data function where superpixels can be computed to construct the graph. A natural choice for such a function is the pixelwise minimizing argument of the data term. This is motivated by the intuition that constant regions of pointwise minimizing disparities likely induce constant regions of the original data term. The features $f_k$ are either given directly as absolute pixelwise disparities, or as output of a stereo network and are then matched globally to combinations of candidate disparities in a given range.
\Cref{fig:stereo} (top) gives a visual impression of the approximation behavior of the graph reduction on an exemplary stereo image. \Cref{fig:stereo} (bottom) visualizes the time vs. the achieved energy values of our method compared to the full stereo matching problem. Note the scale of the x-axis. We can easily reduce the necessary time and memory costs by using the graph-based discretization. Despite of the significant speedup for stereo matching the proposed method still is capable of producing visually pleasing results, as the matching is still computed with respect to all variables, just with an optimally chosen discretization.

\begin{figure}
    \centering
    \includegraphics[width=0.33\textwidth]{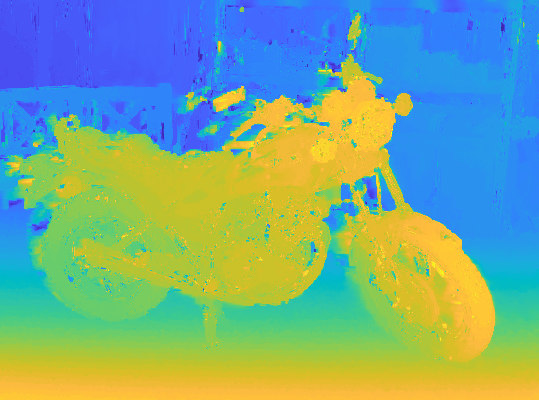}
    \includegraphics[width=0.33\textwidth]{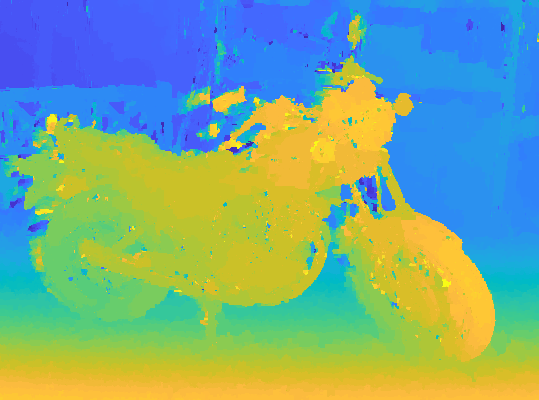}
 \includegraphics[trim={0 0 0 3.1cm},clip, width=0.5\textwidth]{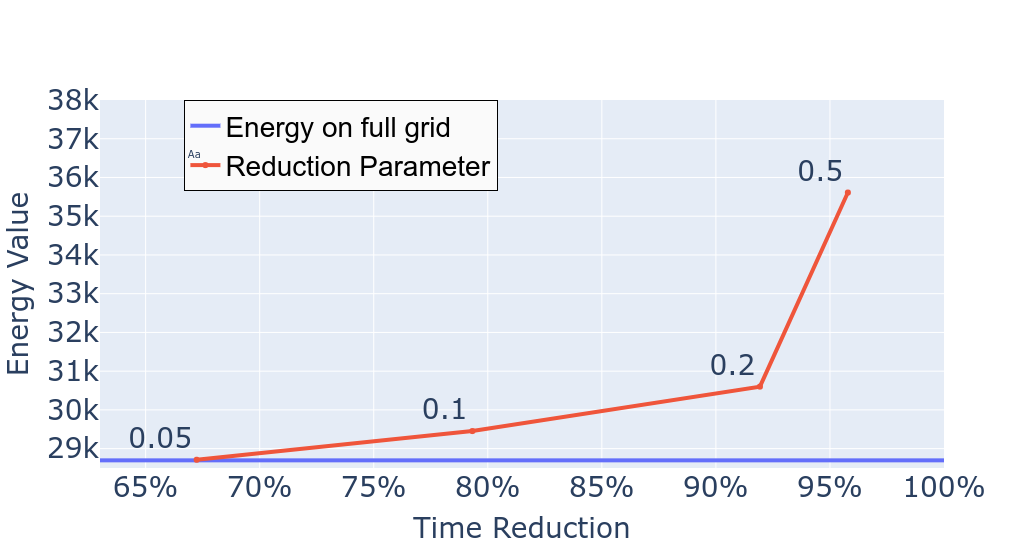}
    \caption{Results on stereo matching baselines. \textit{Top:} Comparison of full (left) and proposed reduced (right) matching. Both methods use sublabels \cite{mollenhoff_sublabel-accurate_2016} between 32 labels. The proposed method uses Cut-Pursuit (with parameter $\alpha_c = 0.1$, a higher parameter corresponds to fewer vertices in the reduced graph) to find a reduced graph, amounting to a time reduction by a factor of 4.8, although the matching quality is near indistinguishable. \textit{Bottom:} Time reduction and energy levels for different parameters $\alpha_c$, showing the granular relationship between graph reduction and difference in energy value of the matching algorithm.}
    \label{fig:stereo}
\end{figure}

\section{Conclusions}
In this work we presented strategies for the efficient realization of convex relaxations by directly moving from geometric properties of minimal partitions solutions to a graph-based discretization. 
We prove that such a graph-based discretization can be constructed in adherence to the global partitioning problem and implementing it on superpixel graphs yields accurate and efficient solutions in practice. 
We further find that using a superpixel approach that is more faithful to minimal surface energies, as the $L^0$-Cut Pursuit algorithm leads to more accurate solutions compared to SLIC and can be well worth the additional effort.
We believe that the proposed methodology can facilitate the use of convex relaxation methods in practical applications, especially if input data is of high-resolution, where memory and computation constraints made these approaches previously infeasible.

\noindent\textbf{Acknowledgements:}\\
This research was supported by the German Research Foundation (DFG) under grant MO 2962/2-1.
\clearpage
\bibliography{zotero_library.bib}

\begin{thebibliography}{52}
\providecommand{\natexlab}[1]{#1}
\providecommand{\url}[1]{\texttt{#1}}
\expandafter\ifx\csname urlstyle\endcsname\relax
  \providecommand{\doi}[1]{doi: #1}\else
  \providecommand{\doi}{doi: \begingroup \urlstyle{rm}\Url}\fi

\bibitem[Achanta et~al.(2012)Achanta, Shaji, Smith, Lucchi, Fua, and
  Süsstrunk]{achanta_slic_2012}
Radhakrishna Achanta, Appu~Shaji Shaji, Kevin Smith, Aurelien Lucchi, Pascal
  Fua, and Sabine Süsstrunk.
\newblock {{SLIC Superpixels Compared}} to {{State}}-of-the-{{Art Superpixel
  Methods}}.
\newblock \emph{IEEE Trans. Pattern Anal. Mach. Intell.}, 34\penalty0
  (11):\penalty0 2274--2282, November 2012.
\newblock ISSN 0162-8828.
\newblock \doi{10.1109/TPAMI.2012.120}.

\bibitem[Ambrosio et~al.(2000)Ambrosio, Fusco, and
  Pallara]{ambrosio_functions_2000}
Luigi Ambrosio, Nicola Fusco, and Diego Pallara.
\newblock \emph{Functions of {{Bounded Variation}} and {{Free Discontinuity
  Problems}}}.
\newblock {Oxford university press}, {Oxford}, 2000.

\bibitem[Attouch et~al.(2006)Attouch, Buttazzo, and
  Michaille]{attouch_variational_2006}
H.~Attouch, Giuseppe Buttazzo, and Gérard Michaille.
\newblock \emph{Variational Analysis in {{Sobolev}} and {{BV}} Spaces:
  Applications to {{PDEs}} and Optimization}.
\newblock {{MPS}}-{{SIAM}} Series on Optimization. {Society for Industrial and
  Applied Mathematics : Mathematical Programming Society}, {Philadelphia},
  2006.
\newblock ISBN 978-0-89871-600-9.

\bibitem[Boykov and {Funka-Lea}(2006)]{boykov_graph_2006-1}
Yuri Boykov and Gareth {Funka-Lea}.
\newblock Graph {{Cuts}} and {{Efficient N}}-{{D Image Segmentation}}.
\newblock \emph{Int. J. Comput. Vis.}, 70\penalty0 (2):\penalty0 109--131,
  November 2006.
\newblock ISSN 0920-5691, 1573-1405.
\newblock \doi{10.1007/s11263-006-7934-5}.

\bibitem[Boykov et~al.(2001)Boykov, Veksler, and Zabih]{boykov_fast_2001}
Yuri Boykov, Olga Veksler, and Ramin Zabih.
\newblock Fast approximate energy minimization via graph cuts.
\newblock \emph{IEEE Trans. Pattern Anal. Mach. Intell.}, 23\penalty0
  (11):\penalty0 1222--1239, November 2001.
\newblock ISSN 0162-8828.
\newblock \doi{10.1109/34.969114}.

\bibitem[Bresson and Chan(2008)]{bresson_non-local_2008-1}
Xavier Bresson and Tony~F. Chan.
\newblock Non-local {{Unsupervised Variational Image Segmentation Models}}.
\newblock {{UCLA CAM}} report, {University of California, Los Angeles}, 2008.

\bibitem[Brox et~al.(2004)Brox, Bruhn, Papenberg, and Weickert]{brox_high_2004}
Thomas Brox, Andrés Bruhn, Nils Papenberg, and Joachim Weickert.
\newblock High {{Accuracy Optical Flow Estimation Based}} on a {{Theory}} for
  {{Warping}}.
\newblock In \emph{Computer {{Vision}} - {{ECCV}} 2004}, volume 3024, pages
  25--36. {Springer Berlin Heidelberg}, {Berlin, Heidelberg}, 2004.
\newblock ISBN 978-3-540-21981-1 978-3-540-24673-2.
\newblock \doi{10.1007/978-3-540-24673-2_3}.

\bibitem[Burger and Osher(2013)]{burger_guide_2013}
Martin Burger and Stanley Osher.
\newblock A {{Guide}} to the {{TV}} zoo.
\newblock In \emph{{{PDE}} Based {{Reconstruction Methods}} in {{Imaging}}},
  number 2090 in Lecture {{Notes}} in {{Mathematics}}. {Springer International
  Publishing}, {Switzerland}, first edition, 2013.

\bibitem[Chambolle and Pock(2011)]{chambolle_first-order_2011}
Antonin Chambolle and Thomas Pock.
\newblock A {{First}}-{{Order Primal}}-{{Dual Algorithm}} for {{Convex
  Problems}} with {{Applications}} to {{Imaging}}.
\newblock \emph{J Math Imaging Vis}, 40\penalty0 (1):\penalty0 120--145, May
  2011.
\newblock ISSN 0924-9907, 1573-7683.
\newblock \doi{10.1007/s10851-010-0251-1}.

\bibitem[Chambolle et~al.(2012)Chambolle, Cremers, and
  Pock]{chambolle_convex_2012}
Antonin Chambolle, Daniel Cremers, and Thomas Pock.
\newblock A {{Convex Approach}} to {{Minimal Partitions}}.
\newblock \emph{SIAM J. Imaging Sci.}, 5\penalty0 (4):\penalty0 1113--1158,
  October 2012.
\newblock ISSN 1936-4954.
\newblock \doi{10.1137/110856733}.

\bibitem[Chan and Vese(2001)]{chan_active_2001}
Tony~F. Chan and Luminita~A. Vese.
\newblock Active contours without edges.
\newblock \emph{IEEE Trans. Image Process.}, 10\penalty0 (2):\penalty0
  266--277, 2001.

\bibitem[Chen et~al.(2016)Chen, Papandreou, Kokkinos, Murphy, and
  Yuille]{chen_deeplab:_2016}
Liang-Chieh Chen, George Papandreou, Iasonas Kokkinos, Kevin Murphy, and
  Alan~L. Yuille.
\newblock {{DeepLab}}: {{Semantic Image Segmentation}} with {{Deep
  Convolutional Nets}}, {{Atrous Convolution}}, and {{Fully Connected CRFs}}.
\newblock \emph{ArXiv160600915 Cs}, June 2016.

\bibitem[Cremers et~al.(2011)Cremers, Pock, Kolev, and
  Chambolle]{cremers_convex_2011}
Daniel Cremers, Thomas Pock, Kalin Kolev, and A.~Chambolle.
\newblock Convex {{Relaxation Techniques}} for {{Segmentation}}, {{Stereo}} and
  {{Multiview Reconstruction}}.
\newblock In \emph{Markov {{Random Fields}} for {{Vision}} and {{Image
  Processing}}}. {MIT Press}, {Boston}, 2011.

\bibitem[Eckstein and Bertsekas(1992)]{eckstein_douglasrachford_1992}
Jonathan Eckstein and Dimitri~P. Bertsekas.
\newblock On the {{Douglas}}—{{Rachford}} splitting method and the proximal
  point algorithm for maximal monotone operators.
\newblock \emph{Mathematical Programming}, 55\penalty0 (1):\penalty0 293--318,
  April 1992.
\newblock ISSN 1436-4646.
\newblock \doi{10.1007/BF01581204}.

\bibitem[Elmoataz et~al.(2008)Elmoataz, Lezoray, and
  Bougleux]{elmoataz_nonlocal_2008}
Abderrahim Elmoataz, Olivier Lezoray, and Sébastien Bougleux.
\newblock Nonlocal {{Discrete Regularization}} on {{Weighted Graphs}}: A
  framework for {{Image}} and {{Manifold Processing}}.
\newblock \emph{IEEE Trans. Image Process.}, 17\penalty0 (7):\penalty0
  1047--1060, 2008.

\bibitem[Gilboa and Osher(2008)]{gilboa_nonlocal_2008}
Guy Gilboa and Stanley Osher.
\newblock Nonlocal {{Operators}} with {{Applications}} to {{Image Processing}}.
\newblock \emph{Multiscale Model. Simul.}, 7\penalty0 (3):\penalty0 1005--1028,
  November 2008.
\newblock ISSN 1540-3459.
\newblock \doi{10.1137/070698592}.

\bibitem[Guinard et~al.(2019)Guinard, Landrieu, Caraffa, and
  Vallet]{guinard_piecewise-planar_2019}
Stephane Guinard, Loic Landrieu, Laurent Caraffa, and Bruno Vallet.
\newblock Piecewise-{{Planar Approximation}} of {{Large 3D Data}} as
  {{Graph}}-{{Structured Optimization}}.
\newblock In \emph{{{ISPRS Annals}} of {{Photogrammetry}}, {{Remote Sensing}}
  and {{Spatial Information Sciences}}}, volume IV-2-W5, pages 365--372.
  {Copernicus GmbH}, May 2019.
\newblock \doi{10.5194/isprs-annals-IV-2-W5-365-2019}.

\bibitem[Horn and Schunck(1981)]{horn_determining_1981}
Berthold K.~P. Horn and Brian~G. Schunck.
\newblock Determining optical flow.
\newblock \emph{Artificial Intelligence}, 17\penalty0 (1):\penalty0 185--203,
  August 1981.
\newblock ISSN 0004-3702.
\newblock \doi{10.1016/0004-3702(81)90024-2}.

\bibitem[Ishikawa and Geiger(1998)]{ishikawa_occlusions_1998}
Hiroshi Ishikawa and Davi Geiger.
\newblock Occlusions, discontinuities, and epipolar lines in stereo.
\newblock In \emph{Computer {{Vision}} — {{ECCV}}'98}, Lecture {{Notes}} in
  {{Computer Science}}, pages 232--248. {Springer Berlin Heidelberg}, 1998.
\newblock ISBN 978-3-540-69354-3.

\bibitem[Kolev and Cremers(2008)]{kolev_integration_2008}
Kalin Kolev and Daniel Cremers.
\newblock Integration of {{Multiview Stereo}} and {{Silhouettes Via Convex
  Functionals}} on {{Convex Domains}}.
\newblock In \emph{Computer {{Vision}} – {{ECCV}} 2008}, Lecture {{Notes}} in
  {{Computer Science}}, pages 752--765, {Berlin, Heidelberg}, 2008. {Springer}.
\newblock ISBN 978-3-540-88682-2.
\newblock \doi{10.1007/978-3-540-88682-2_57}.

\bibitem[Kolmogorov and Zabin(2004)]{kolmogorov_what_2004}
Vladimir Kolmogorov and Ramin Zabin.
\newblock What energy functions can be minimized via graph cuts?
\newblock \emph{IEEE Trans. Pattern Anal. Mach. Intell.}, 26\penalty0
  (2):\penalty0 147--159, February 2004.
\newblock ISSN 0162-8828.
\newblock \doi{10.1109/TPAMI.2004.1262177}.

\bibitem[Landrieu and Obozinski(2016)]{landrieu_cut_2016}
Loic Landrieu and Guillaume Obozinski.
\newblock Cut {{Pursuit}}: Fast algorithms to learn piecewise constant
  functions.
\newblock In \emph{Proceedings of the 19th {{International Conference}} on
  {{Artificial Intelligence}} and {{Statistics}}}, volume~51 of
  \emph{Proceedings of {{Machine Learning Research}}}, pages 1384--1393,
  {Cadiz, Spain}, April 2016. {PMLR}.

\bibitem[Landrieu and Obozinski(2017)]{landrieu_cut_2017}
Loic Landrieu and Guillaume Obozinski.
\newblock Cut {{Pursuit}}: {{Fast Algorithms}} to {{Learn Piecewise Constant
  Functions}} on {{General Weighted Graphs}}.
\newblock \emph{SIAM J. Imaging Sci.}, 10\penalty0 (4):\penalty0 1724--1766,
  January 2017.
\newblock \doi{10.1137/17M1113436}.

\bibitem[Landrieu and Simonovsky(2017)]{landrieu_large-scale_2017}
Loic Landrieu and Martin Simonovsky.
\newblock Large-scale {{Point Cloud Semantic Segmentation}} with {{Superpoint
  Graphs}}.
\newblock \emph{ArXiv171109869 Cs}, November 2017.

\bibitem[Laude et~al.(2016)Laude, Möllenhoff, Moeller, Lellmann, and
  Cremers]{laude_sublabel-accurate_2016-1}
Emanuel Laude, Thomas Möllenhoff, Michael Moeller, Jan Lellmann, and Daniel
  Cremers.
\newblock Sublabel-{{Accurate Convex Relaxation}} of {{Vectorial Multilabel
  Energies}}.
\newblock In \emph{Computer {{Vision}} – {{ECCV}} 2016}, Lecture {{Notes}} in
  {{Computer Science}}, pages 614--627. {Springer, Cham}, October 2016.
\newblock ISBN 978-3-319-46447-3 978-3-319-46448-0.
\newblock \doi{10.1007/978-3-319-46448-0_37}.

\bibitem[Lellmann et~al.(2009)Lellmann, Kappes, Yuan, Becker, and
  Schnörr]{lellmann_convex_2009}
Jan Lellmann, Jörg Kappes, Jing Yuan, Florian Becker, and Christoph Schnörr.
\newblock Convex {{Multi}}-class {{Image Labeling}} by
  {{Simplex}}-{{Constrained Total Variation}}.
\newblock In \emph{Scale {{Space}} and {{Variational Methods}} in {{Computer
  Vision}}}, Lecture {{Notes}} in {{Computer Science}}, pages 150--162.
  {Springer Berlin Heidelberg}, 2009.
\newblock ISBN 978-3-642-02256-2.

\bibitem[Lellmann et~al.(2013)Lellmann, Strekalovskiy, Koetter, and
  Cremers]{lellmann_total_2013}
Jan Lellmann, Evgeny Strekalovskiy, Sabrina Koetter, and Daniel Cremers.
\newblock Total {{Variation Regularization}} for {{Functions}} with {{Values}}
  in a {{Manifold}}.
\newblock In \emph{Proceedings of the 2013 {{IEEE International Conference}} on
  {{Computer Vision}}}, {{ICCV}} '13, pages 2944--2951, {Washington, DC, USA},
  2013. {IEEE Computer Society}.
\newblock ISBN 978-1-4799-2840-8.
\newblock \doi{10.1109/ICCV.2013.366}.

\bibitem[Martin et~al.(2001)Martin, Fowlkes, Tal, and
  Malik]{martin_database_2001}
David Martin, Charless Fowlkes, Doron Tal, and Jitendra Malik.
\newblock A {{Database}} of {{Human Segmented Natural Images}} and its
  {{Application}} to {{Evaluating Segmentation Algorithms}} and {{Measuring
  Ecological Statistics}}.
\newblock In \emph{Proceedings of 8th {{International Conference}} on
  {{Computer Vision}}}, volume~2, pages 416--423, July 2001.

\bibitem[Mumford and Shah(1989)]{mumford_optimal_1989}
David Mumford and Jayant Shah.
\newblock Optimal approximations by piecewise smooth functions and associated
  variational problems.
\newblock \emph{Commun. Pure Appl. Math.}, 42\penalty0 (5):\penalty0 577--685,
  1989.

\bibitem[Möllenhoff and Cremers(2017)]{mollenhoff_sublabel-accurate_2017}
Thomas Möllenhoff and Daniel Cremers.
\newblock Sublabel-{{Accurate Discretization}} of {{Nonconvex
  Free}}-{{Discontinuity Problems}}.
\newblock \emph{Proc. IEEE Int. Conf. Comput. Vis.}, pages 1183--1191, August
  2017.
\newblock \doi{10.1109/ICCV.2017.134}.

\bibitem[Möllenhoff et~al.(2016)Möllenhoff, Laude, Moeller, Lellmann, and
  Cremers]{mollenhoff_sublabel-accurate_2016}
Thomas Möllenhoff, Emanuel Laude, Michael Moeller, Jan Lellmann, and Daniel
  Cremers.
\newblock Sublabel-{{Accurate Relaxation}} of {{Nonconvex Energies}}.
\newblock In \emph{Proceedings of the {{IEEE Conference}} on {{Computer
  Vision}} and {{Pattern Recognition}}}, pages 3948--3956, 2016.
\newblock \doi{10.1109/CVPR.2016.428}.

\bibitem[Nieuwenhuis et~al.(2013)Nieuwenhuis, Töppe, and
  Cremers]{nieuwenhuis_survey_2013}
Claudia Nieuwenhuis, Eno Töppe, and Daniel Cremers.
\newblock A {{Survey}} and {{Comparison}} of {{Discrete}} and {{Continuous
  Multi}}-label {{Optimization Approaches}} for the {{Potts Model}}.
\newblock \emph{Int J Comput Vis}, 104\penalty0 (3):\penalty0 223--240,
  September 2013.
\newblock ISSN 1573-1405.
\newblock \doi{10.1007/s11263-013-0619-y}.

\bibitem[Pock and Chambolle(2011)]{pock_diagonal_2011}
Thomas Pock and Antonin Chambolle.
\newblock Diagonal preconditioning for first order primal-dual algorithms in
  convex optimization.
\newblock In \emph{2011 {{International Conference}} on {{Computer Vision}}},
  pages 1762--1769, November 2011.
\newblock \doi{10.1109/ICCV.2011.6126441}.

\bibitem[Pock et~al.(2008)Pock, Schoenemann, Graber, Bischof, and
  Cremers]{pock_convex_2008}
Thomas Pock, Thomas Schoenemann, Gottfried Graber, Horst Bischof, and Daniel
  Cremers.
\newblock A convex formulation of continuous multi-label problems.
\newblock In \emph{European Conference on Computer Vision}, pages 792--805.
  {Springer}, 2008.

\bibitem[Pock et~al.(2009)Pock, Chambolle, Cremers, and
  Bischof]{pock_convex_2009}
Thomas Pock, Antonin Chambolle, Daniel Cremers, and Horst Bischof.
\newblock A convex relaxation approach for computing minimal partitions.
\newblock In \emph{Computer {{Vision}} and {{Pattern Recognition}}, 2009.
  {{CVPR}} 2009. {{IEEE Conference}} On}, pages 810--817. {IEEE}, 2009.

\bibitem[Pock et~al.(2010)Pock, Cremers, Bischof, and
  Chambolle]{pock_global_2010}
Thomas Pock, Daniel Cremers, Horst Bischof, and Antonin Chambolle.
\newblock Global {{Solutions}} of {{Variational Models}} with {{Convex
  Regularization}}.
\newblock \emph{SIAM J. Imaging Sci.}, 3\penalty0 (4):\penalty0 1122--1145,
  January 2010.
\newblock \doi{10.1137/090757617}.

\bibitem[Potts(1952)]{potts_generalized_1952}
Renfrey~Burnard Potts.
\newblock Some generalized order-disorder transformations.
\newblock In \emph{Mathematical Proceedings of the Cambridge Philosophical
  Society}, volume~48, pages 106--109. {Cambridge Univ Press}, 1952.

\bibitem[Ranftl et~al.(2012)Ranftl, Gehrig, Pock, and
  Bischof]{ranftl_pushing_2012}
Rene Ranftl, Stefan Gehrig, Thomas Pock, and Horst Bischof.
\newblock Pushing the limits of stereo using variational stereo estimation.
\newblock In \emph{Intelligent {{Vehicles Symposium}} ({{IV}}), 2012 {{IEEE}}},
  pages 401--407. {IEEE}, 2012.

\bibitem[Ranftl et~al.(2013)Ranftl, Pock, and Bischof]{ranftl_minimizing_2013}
Rene Ranftl, Thomas Pock, and Horst Bischof.
\newblock Minimizing {{TGV}}-{{Based Variational Models}} with {{Non}}-convex
  {{Data Terms}}.
\newblock In \emph{Scale {{Space}} and {{Variational Methods}} in {{Computer
  Vision}}}, Lecture {{Notes}} in {{Computer Science}}, pages 282--293.
  {Springer Berlin Heidelberg}, June 2013.
\newblock ISBN 978-3-642-38266-6 978-3-642-38267-3.
\newblock \doi{10.1007/978-3-642-38267-3_24}.

\bibitem[Rudin et~al.(1992)Rudin, Osher, and Fatemi]{rudin_nonlinear_1992}
Leonid~I. Rudin, Stanley Osher, and Emad Fatemi.
\newblock Nonlinear total variation based noise removal algorithms.
\newblock \emph{Physica D: Nonlinear Phenomena}, 60\penalty0 (1):\penalty0
  259--268, November 1992.
\newblock ISSN 0167-2789.
\newblock \doi{10.1016/0167-2789(92)90242-F}.

\bibitem[Scharstein et~al.(2014)Scharstein, Hirschmüller, Kitajima, Krathwohl,
  Nešić, Wang, and Westling]{scharstein_high-resolution_2014}
Daniel Scharstein, Heiko Hirschmüller, York Kitajima, Greg Krathwohl, Nera
  Nešić, Xi~Wang, and Porter Westling.
\newblock High-{{Resolution Stereo Datasets}} with {{Subpixel}}-{{Accurate
  Ground Truth}}.
\newblock In \emph{Pattern {{Recognition}}}, Lecture {{Notes}} in {{Computer
  Science}}, pages 31--42. {Springer, Cham}, September 2014.
\newblock ISBN 978-3-319-11751-5 978-3-319-11752-2.
\newblock \doi{10.1007/978-3-319-11752-2_3}.

\bibitem[Souiai et~al.(2015)Souiai, Oswald, Kee, Kim, Pollefeys, and
  Cremers]{souiai_entropy_2015}
Mohamed Souiai, Martin~R. Oswald, Youngwook Kee, Junmo Kim, Marc Pollefeys, and
  Daniel Cremers.
\newblock Entropy {{Minimization}} for {{Convex Relaxation Approaches}}.
\newblock In \emph{Proceedings of the {{IEEE International Conference}} on
  {{Computer Vision}}}, pages 1778--1786, 2015.

\bibitem[Strekalovskiy and Cremers(2014)]{strekalovskiy_real-time_2014}
Evgeny Strekalovskiy and Daniel Cremers.
\newblock Real-{{Time Minimization}} of the {{Piecewise Smooth Mumford}}-{{Shah
  Functional}}.
\newblock In \emph{Computer {{Vision}} – {{ECCV}} 2014}, Lecture {{Notes}} in
  {{Computer Science}}, pages 127--141. {Springer, Cham}, September 2014.
\newblock ISBN 978-3-319-10604-5 978-3-319-10605-2.
\newblock \doi{10.1007/978-3-319-10605-2_9}.

\bibitem[Strekalovskiy et~al.(2011)Strekalovskiy, Goldluecke, and
  Cremers]{strekalovskiy_tight_2011}
Evgeny Strekalovskiy, Bastian Goldluecke, and Daniel Cremers.
\newblock Tight convex relaxations for vector-valued labeling problems.
\newblock In \emph{2011 {{International Conference}} on {{Computer Vision}}},
  pages 2328--2335, November 2011.
\newblock \doi{10.1109/ICCV.2011.6126514}.

\bibitem[Strekalovskiy et~al.(2014)Strekalovskiy, Chambolle, and
  Cremers]{strekalovskiy_convex_2014}
Evgeny Strekalovskiy, Antonin Chambolle, and Daniel Cremers.
\newblock Convex {{Relaxation}} of {{Vectorial Problems}} with {{Coupled
  Regularization}}.
\newblock \emph{SIAM J. Imaging Sci.}, 7\penalty0 (1):\penalty0 294--336,
  January 2014.
\newblock ISSN 1936-4954.
\newblock \doi{10.1137/130908348}.

\bibitem[Tenbrinck et~al.(2019)Tenbrinck, Gaede, and
  Burger]{tenbrinck_variational_2019}
Daniel Tenbrinck, Fjedor Gaede, and Martin Burger.
\newblock Variational {{Graph Methods}} for {{Efficient Point Cloud
  Sparsification}}.
\newblock \emph{ArXiv190302858 Cs Math}, March 2019.

\bibitem[Uziel et~al.(2019)Uziel, Ronen, and Freifeld]{uziel_bayesian_2019}
Roy Uziel, Meitar Ronen, and Oren Freifeld.
\newblock Bayesian {{Adaptive Superpixel Segmentation}}.
\newblock In \emph{Proceedings of the {{IEEE International Conference}} on
  {{Computer Vision}}}, pages 8470--8479, 2019.

\bibitem[Werlberger et~al.(2011)Werlberger, Unger, Pock, and
  Bischof]{werlberger_efficient_2011}
Manuel Werlberger, Markus Unger, Thomas Pock, and Horst Bischof.
\newblock Efficient {{Minimization}} of the {{Non}}-local {{Potts Model}}.
\newblock In \emph{Scale {{Space}} and {{Variational Methods}} in {{Computer
  Vision}}}, Lecture {{Notes}} in {{Computer Science}}, pages 314--325.
  {Springer Berlin Heidelberg}, May 2011.
\newblock ISBN 978-3-642-24784-2 978-3-642-24785-9.
\newblock \doi{10.1007/978-3-642-24785-9_27}.

\bibitem[Winn et~al.(2005)Winn, Criminisi, and Minka]{winn_object_2005}
John~M Winn, Antonio Criminisi, and Thomas~P Minka.
\newblock Object categorization by learned universal visual dictionary.
\newblock In \emph{Tenth {{IEEE International Conference}} on {{Computer
  Vision}} ({{ICCV}}'05) {{Volume}} 1}, volume~2, pages 1800--1807 Vol. 2,
  October 2005.
\newblock \doi{10.1109/ICCV.2005.171}.

\bibitem[Zach et~al.(2007)Zach, Pock, and Bischof]{zach_duality_2007}
Christopher Zach, Thomas Pock, and Horst Bischof.
\newblock A {{Duality Based Approach}} for {{Realtime TV}}-{{L1 Optical Flow}}.
\newblock In \emph{Pattern {{Recognition}}}, pages 214--223. {Springer, Berlin,
  Heidelberg}, September 2007.
\newblock \doi{10.1007/978-3-540-74936-3_22}.

\bibitem[Zach et~al.(2008)Zach, Gallup, Frahm, and Niethammer]{zach_fast_2008}
Christopher Zach, David Gallup, Jan-Michael Frahm, and Marc Niethammer.
\newblock Fast {{Global Labeling}} for {{Real}}-{{Time Stereo Using Multiple
  Plane Sweeps}}.
\newblock In \emph{Vision, {{Modeling}}, and {{Visualization}}}, pages pp.
  243--252, {Amsterdam, The Netherlands}, August 2008. {IOS Press}.

\bibitem[Zach et~al.(2012)Zach, Häne, and Pollefeys]{zach_what_2012}
Christopher Zach, Christian Häne, and Marc Pollefeys.
\newblock What is optimized in tight convex relaxations for multi-label
  problems?
\newblock In \emph{2012 {{IEEE Conference}} on {{Computer Vision}} and
  {{Pattern Recognition}}}, pages 1664--1671, June 2012.
\newblock \doi{10.1109/CVPR.2012.6247860}.

\end{thebibliography}

\appendix

\section{Proof of Proposition 1}\label{sec:proof}
We intend to show Proposition 1 by first showing a lemma for functionals of the form
\begin{align}\tag{5}\label{eq:functional}
\begin{split}
      F(u) = &\int_{\Omega \setminus J_u} \Phi(x, u(x), \nabla u(x)) \ dx  \\
    + &\int_{J_u} \kappa\left(|u^+-u^-|\right)|\nu_u| d\mathcal{H}^{d-1},  
\end{split}
\end{align}
for $u \in SBV(\Omega, \R^L)$ under constraints $C$ given as
\begin{align*}
 C = \big\{ u ~ \big| ~ u_k(x) \in [0,1],~ \sum_{k=1}^L u_k(x) = 1 \big\}.
 \end{align*}
\begin{lemma}\label{lemma1}
Assume a discretization $\Pi$ and its assorted partitions $P_i$ to be given.
Let $u^*$ be a minimizer to the continuous problem \cref{eq:functional}. If the jump-set $J_{u^*}$ of $u^*$ is a subset of the jump set of $\Pi$ given as the boundaries $\cup_{(i,j)\in E} \Gamma_{ij}$, then
\begin{equation*}
    \min_{u\in C} F(u) = \min_{c\in C_\Pi} F_\Pi(c),
\end{equation*}
for the discrete energy $F_\Pi = K_\Pi + R_\Pi$,
i.e. the continuous minimum $F(u^*)$ is equal to the minimum $F_\Pi(c^*)$ of the discrete energy of $F_\Pi$ under the constraints $C_\Pi$.
\end{lemma}
\begin{proof}
As defined in Section 3.2 of the main paper, we consider the space $SBV_\Pi(\Omega, \R^L)$ of functions in $SBV(\Omega, \R^L)$ that are piecewise-constant on partition $\Pi$. From the assumption that $J_{u^*}$ is a subset of the jump set of $\Pi$, given by $\bigcup_{(i,j)\in E} \Gamma_{ij}$, we deduce $u^* \in SBV_\Pi(\Omega, \R^L)$. 
For a partition $\Pi$ define
\begin{equation*}
    {\Xi}_\Pi = \{ u \in C \mid u \in SBV_\Pi(\Omega, \R^L) \},
\end{equation*}
for 
\begin{align*}
 C = \big\{ u ~ \big| ~ u_k(x) \in [0,1],~ \sum_{k=1}^L u_k(x) = 1 \big\}.
 \end{align*}
Then $\Pi' \leq \Pi$, where "$\leq$" refers to the partial order on partitions meaning $\Pi'$ is a finer partition than $\Pi$. This implies $\Xi_\Pi \subseteq \Xi_{\Pi'}$. For $u \in C_\Pi$ and the according $c$ from Section 3.2 we already have shown $F_\Pi ( c ) = F ( u )$.

Hence, $u^* \in \Xi_\Pi$ allows us to write
\begin{align*}
    \min_{c \in C_\Pi} F_\Pi \left( c \right) & = \min_{\tilde u \in \Xi_\Pi} F ( \tilde u ) \\
    & \leq F \left( u^* \right) = \min_{u \in C} F ( u ).
\end{align*}
Equality now follows due to $\Xi_\Pi \subseteq C$ from
\begin{align*}
    \min_{u \in C} F ( u ) \leq \min_{\tilde u \in \Xi_\Pi} F ( \tilde u ) = \min_{c \in C_\Pi} F_\Pi \left( c\right).
\end{align*}
\end{proof}

Now we can find proposition 1 as a simply corollary. The minimal partitions problem
\begin{equation}\tag{2}
    \min_{\lbrace u_k\rbrace_{k=1}^L} \ \sum_{k=1}^L \ \int_\Omega -f_k(x)u_k(x) + \int_\Omega |Du_k|,
\end{equation}
is a special case of \cref{eq:functional} by choosing the data term via 
\begin{equation}\tag{7}
    K_\Pi(c)=\sum_{k=1}^L \sum_{i=1}^M ~ (c_i)_k \int_{P_i} f_k(x) \ dx  = \sum_{i=1}^M ~ \langle c_i, \tilde{f}_k \rangle. 
\end{equation}
and setting $\kappa= \operatorname{Id}$ for the regularization term.

\setcounter{proposition}{0}
\begin{proposition}
Assume a discretization $\Pi$ and its assorted partitions $P_i$ to be given.
Let $u^*$ be a minimizer to the continuous problem \cref{eq:minimal_partitions_binary} for given potentials $f_k$. If the jump-set $J_{u^*}$ of $u^*$ is a subset of the jump set of $\Pi$ given as the boundaries $\cup_{(i,j)\in E} \Gamma_{ij}$, then
\begin{equation*}
    \min_{u\in C} F(u) = \min_{c\in C_\Pi} F_\Pi(c),
\end{equation*}
for the discrete energy $F_\Pi = K_\Pi + R_\Pi$,
i.e. the continuous minimum $F(u^*)$ is equal to the minimum $F_\Pi(c^*)$ of the discrete energy of $F_\Pi$ under the constraints $C_\Pi$.
\end{proposition}
\begin{proof}
Apply \cref{lemma1} to \cref{eq:minimal_partitions_binary}.
\end{proof}



\section{Algorithmic Details}

\begin{figure*}
    \centering
    \includegraphics[width=0.32\textwidth]{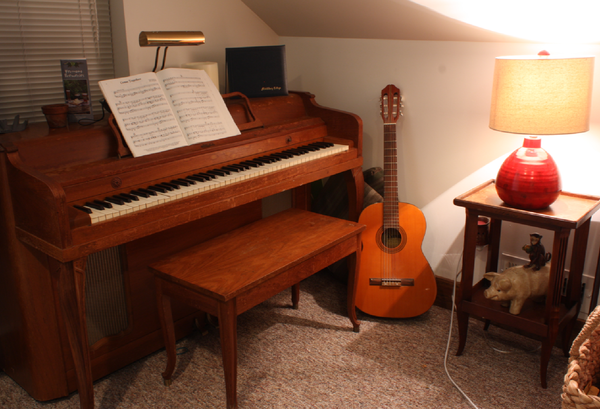}
    \includegraphics[width=0.32\textwidth]{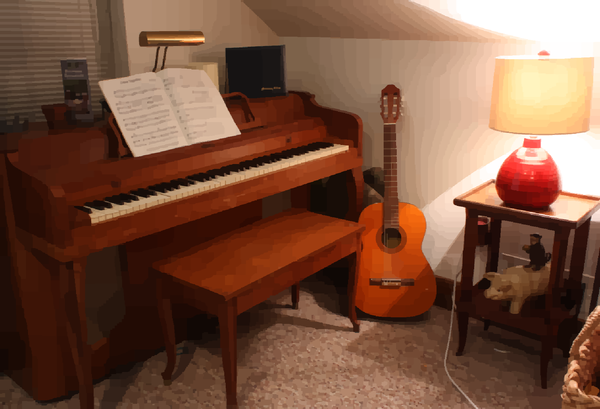}
    \includegraphics[width=0.32\textwidth]{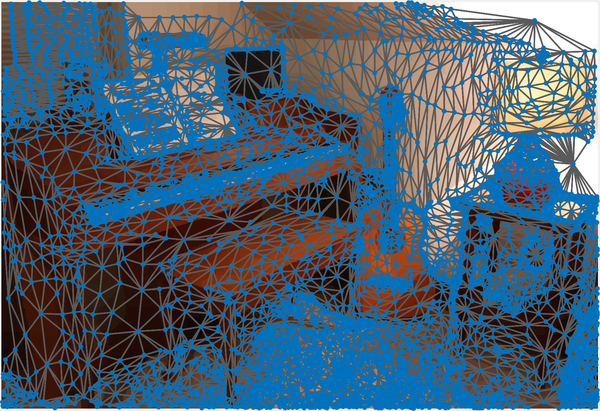}
    \caption{Reduction of the raw image (left) to 6031 nodes in the graph structure (Right). In the middle we reproject the graph structure onto the original image, visualizing the high fidelity of the representation, even as the number of nodes reduces to $0.45\%$ of the full grid.}
    \label{fig:compare}
\end{figure*}
For implementation reference we replicate some parts of the $L^0$ Cut-Pursuit \cite{landrieu_cut_2016} variant of \cite{tenbrinck_variational_2019} in the continuous setting.

To obtain a good trade-off between having as few segments as possible but still constructing a partition whose jump set is a super set of the jump set of a minimizer $u^*$, we exploit
a modification of the Cut-Pursuit (CP) algorithm of \cite{landrieu_cut_2016} discussed in \cite{tenbrinck_variational_2019}. In \cite{landrieu_cut_2016} Landrieu and Obozinski develop an approach to solve total variation problems \cite{rudin_nonlinear_1992,burger_guide_2013} with an alternating method solving graph cuts and reduced problems on the smaller graphs generated by these cuts. This is an efficient method with superior performance compared to more classical optimization methods as primal-dual \cite{chambolle_first-order_2011} or Douglas-Rachford \cite{eckstein_douglasrachford_1992} algorithms minimizing the total variation problem. The Cut-Pursuit algorithm can be further extended to a variant minimizing the $L^0$ norm of the graph gradient. This strategy is able to quickly return approximate solutions to partitioning problems for a relatively high number of partitions compared to the number of potentials $L$ we consider. The final number of partitions depends on regularization parameters and is data-dependent.
 In \cite{tenbrinck_variational_2019} Tenbrinck et. al. modify the Cut-Pursuit for $L^0$ by simplifying the algorithm to alternating between a graph cut and solving the data term separately on each generated partitions.  We will denote this method as $L^0$-Cut-Pursuit ($L^0$-CP). We discuss this algorithm in a continuous setting 
 with an $L^2$ data fidelity term, resulting in the following alternating algorithm:
For a given function $u^k \in SBV(\Omega, \R)$ and some given data $g \in L^1(\Omega, \R)$, one step of the algorithm consists of the two alternating optimization steps. The first one is
{\small
\begin{align} \label{eq:partProb}
    B^{k+1} = &\operatornamewithlimits{arg min}_{B \subset \Omega} \ \int_\Omega (u^k(x) - g(x) )1_B(x) \ dx + \alpha_c
    \int_\Omega |D 1_B|,
\end{align}}
where $1_B$ denotes a characteristic function on $B$.
This set minimization in \cref{eq:partProb} is binary and thus globally solvable.
Then compute $\Pi^{k+1}$ as the connected components of $B^{k+1}$ and, in a second step, find the mean over every partition,
\begin{align} \label{eq:reducedProb}
    c^{k+1}_i = \frac{1}{|P_i|}\int_{P_i} g(x) ~dx .
\end{align}
From the values $c^{k+1}_i$ of the partitions $P_i \in \Pi^{k+1}$, we can compute the continuous solution via
\begin{align}\label{eq:cktouk}
    u^{k+1} = \sum_{P_i \in \Pi^{k+1}} c^{k+1}_i~ 1_{P_i}. 
\end{align}
Note that such an algorithm operates on $SBV(\Omega, \R)$ rather than $SBV(\Omega, \R^L)$ in order to be much more efficient, particularly for large $L$. 
\begin{algorithm}
\SetAlgoLined
\KwResult{Discretization $\Pi$  }
 $c^0 \leftarrow \operatorname{mean}(g)$\;
 $\Pi^0 \leftarrow \{\Omega\}$\;
 \While{$\Pi^k \neq \Pi^{k+1}$}{
  $B^{k+1} \leftarrow \text{Solve } \cref{eq:partProb}$ for given $u^k$\;
  $\Pi^{k+1} \leftarrow $ connected components of $B^{k+1}$\;
  $c^{k+1} \leftarrow  \text{Solve } \cref{eq:reducedProb}$ for given $\Pi^{k+1}$\;
  $u^{k+1} \leftarrow$ Compute as in \cref{eq:cktouk}\;
  $k\leftarrow k + 1$\;
 }
 $\Pi \leftarrow \Pi^{k+1}$
 \caption{L0-Pursuit from \cite{tenbrinck_variational_2019}}
 \label{alg:cut_l0}
\end{algorithm}
In comparison to a naive subsampling on a grid (left) and the SLIC superpixels from \cite{achanta_slic_2012}, the $L^0$-CP generates less uniformly-sized regions, allowing to combine large constant regions into a single node in a graph and thus being well suited for an efficient coarsification with accurate edges. 
\Cref{alg:cut_l0} shows the  steps that this algorithm follows for further clarification.

\subsection{Implementation}
On a discrete image grid generated by sensor data, \cref{eq:partProb} becomes a binary partitioning on a discrete graph, which can be solved efficiently by a \emph{maxflow} algorithm, e.g. Boykov-Kolmogorov \cite{boykov_fast_2001}. In the end variants, such as $L^1$-Cut-Pursuit \cite{landrieu_cut_2016} or using a real-time Mumford-Shah such as \cite{strekalovskiy_real-time_2014} would also be possible candidates to find a candidate partition, yet we did not find these variants to yield either sufficient speed or sufficient accuracy around edges to be applicable - for algorithms that do not explicitly track the partitioning, the partition also has to be computed from the final result in an additional post-processing step.

When applying the Cut-Pursuit algorithm, we first need to consider which data will be used for $g$, the input to the Cut-Pursuit algorithm. A straightforward approach is to set
\begin{equation}
    g(x) = \operatornamewithlimits{argmin}_k ~ f_k(x)
\end{equation}
for the given label potentials, but especially for segmentation, using the color or grayscale image data directly is also reasonable under the assumption that piecewise constant objects in the RGB image correspond belong to separate labels, as done for algorithms such as SLIC.

We have chosen the \texttt{search-trees} implementation of \cite{boykov_graph_2006-1} to solve the discrete version of the binary partition problem stated in \cref{eq:partProb}. This can be done by reformulating the energy into a flow-graph structure with two additional terminal nodes \emph{sink} and \emph{source}. How to assign the right capacities to the edges can be taken from \cite{landrieu_cut_2016}  or \cite{tenbrinck_variational_2019}. A significant bottleneck of this straightforward \textit{maxflow} implementation is that the computation is difficult to parallelize. Thus, the computational time can increase drastically for very large images or other input data. For real-time applications with access to parallelization via GPUs or CPUs with sufficient cores we would recommend porting the entire pipeline into a single framework and using a primal dual algorithm with diagonal preconditioned stepsizes as in \cite{pock_diagonal_2011} not only for the minimal partitions problem but also the binary cuts. Especially running both subroutines on the GPU is potentially highly beneficial for large images. On the other hand, solving the binary cut with a primal-dual algorithm only approximates the solution in finite time and convergence criteria have to be chosen carefully to guarantee accurate results. In contrast \textit{maxflow} termination criteria are straightforward, which is why we focus on \textit{maxflow} in this work, aside from its applicability to weaker hardware with low specifications. 

\section{Experimental Setup}
We implement the graph-structured optimization of the convexified partition problem via a primal-dual algorithm \cite{chambolle_first-order_2011} with diagonal preconditioning \cite{pock_diagonal_2011}. The preconditioning allows us to reconcile the step sizes of the algorithm with the varying sizes of the graph partitions.
We use the implementation of this algorithm from \url{https://github.com/tum-vision/prost}, which conducts GPU computations with a Matlab wrapper. The $L^0$-CP implementation is written in Matlab using just the internal \textit{maxflow} implementation.
For the usecase study we use colors and coordinates as features to be classified via a 2-hidden-layer fully network with batchnorm and ReLU activations with 6 and 12 hidden neurons. Due to the tiny architecture and the few scribbles, the training of the network takes 18 seconds on a Laptop CPU (without fine-tuning the hyperparameters), and inferring pixelwise unaries on the entire $1440 \times 1920$ pixel image takes less than a second.

For the comparison of the inset table (Table 1) we consider three example images, 1. "cedar.bmp"\cite{winn_object_2005}, 2. "fish.jpg"\cite{martin_database_2001}, 3. "bin.png"\cite{scharstein_high-resolution_2014}, computed for a multi-label segmentation with $L=16$ labels. \textit{Reduction Rate} is the ratio between the number of segments to the full number of nodes. \textit{Time save} describes the ratio of time that was saved by the graph discretization and \textit{Energy offset} the ratio of energy mismatch.
Note that we always denote the measured the time as the sum of the time used for the reduction method and the computational time to solve the label problem.

\section{Superpixel-Sublabel Stereo Lifting}
When using a sublabel-based stereo estimation, then attention has to be paid to the treatment of the data term, as it does not bear the linear structure with to respect to the label coefficients anymore and hence interferes with the direct application of the proposed graph reduction. A closer look on the stereo problem formulation of \cite{mollenhoff_sublabel-accurate_2016} however reveals that the data term in between two neighboring labels is formed by calculating the convex envelope over the finitely sampled disparity costs. To be more precise, the stereo matching cost can be regarded as a in general non-linear data term $f_x \colon \R^L \to \R$ for each $x \in \Omega$. It directly operates on the lifted variable function $u \colon \Omega \to \R^L$. The data term is then relaxed to $f_x^{\ast\ast}$ for each $x$. This eventually amounts to a piecewise linear data term for each interval. Summing the data terms over superpixels, however, even yields a tighter convex approximation as
\begin{align}
\sum_{i=1}^M \left( \int_{P_i} f_x dx \right)^{**} \geq \int_\Omega f^{**} _x dx,
\end{align}
where the left data term is the one effectively used when applying the method from \cite{mollenhoff_sublabel-accurate_2016} to superpixels as discussed.

\section{Further plots}
\Cref{fig:compare} shows the fidelity of the L0-Cut Pursuit representation of an RGB image. The reduction $0.45\%$ in comparison to the full grid is hardly noticable without zooming in.
\Cref{fig:even_more_stats_psnr_ssim} shows a variant of Figure 6 in the paper. We visualize PSNR / SSIM / DICE values for the cartooning problem. These are computed by reassembling the output image from the piecewise constant segmentation and comparing it to the input image.

\begin{figure*}
    \centering
    \includegraphics[width=0.44\textwidth]{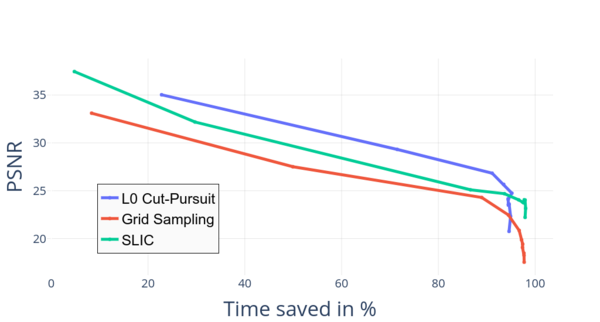}
    \includegraphics[width=0.44\textwidth]{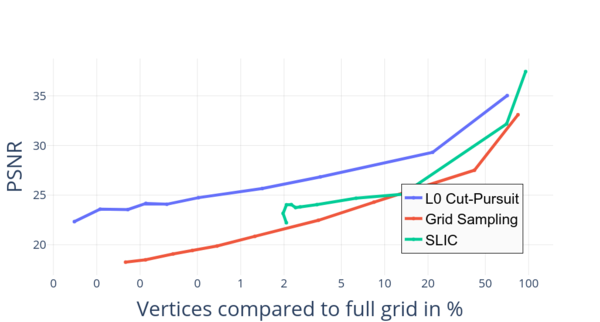}\\
    \includegraphics[width=0.44\textwidth]{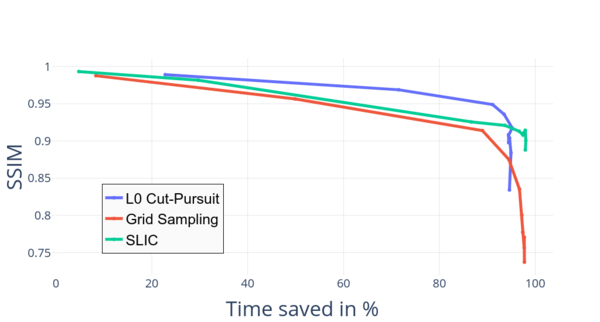}
    \includegraphics[width=0.44\textwidth]{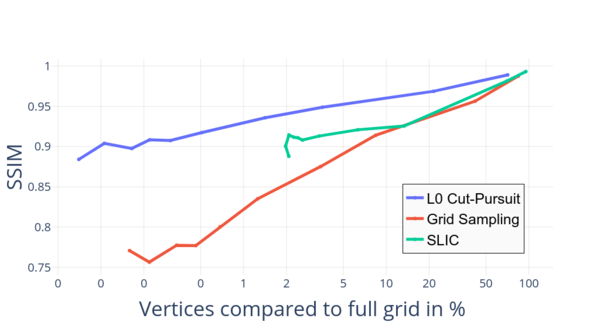}\\
    \includegraphics[width=0.44\textwidth]{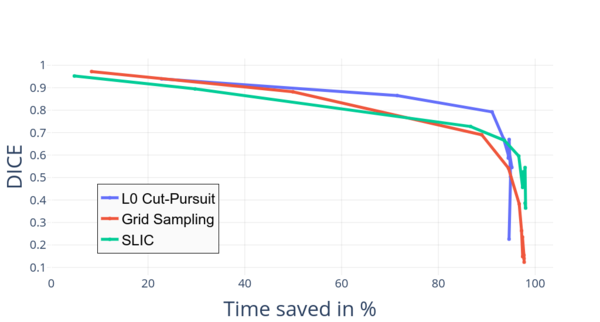}
    \includegraphics[width=0.44\textwidth]{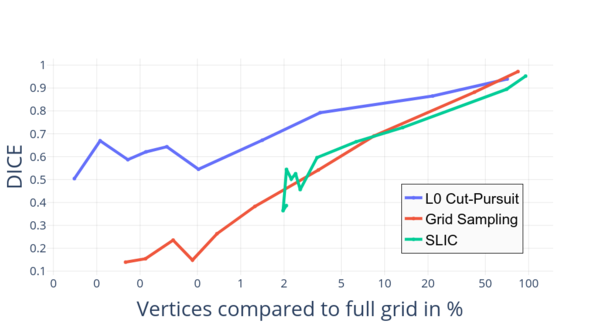}
    \caption{Computing the minimal partition on the graph is much more efficient than grid subsampling or SLIC superpixels. Left: Time saved vs 100\% on the full grid plotted vs PSNR/SSIM/DICE score of the segmentation vs the full grid segmentation. Right: The number of nodes compared also compared to the PSNR/SSIM/DICE score of the segmentation.}
    \label{fig:even_more_stats_psnr_ssim}
\end{figure*}

\clearpage

\end{document}